\documentclass[sigconf,nonacm,pbalance]{acmart}
\AtBeginDocument{%
  }

\newcommand{\mathbbm}[1]{\text{\usefont{U}{DSSerif}{m}{n}#1}}
\usepackage{bm}
\usepackage{thmtools}
\usepackage{mathtools}
\usepackage[capitalise,noabbrev]{cleveref}
\crefformat{equation}{(#2#1#3)}
\usepackage{soul}
\usepackage{autonum}
\usepackage{stackrel}
\usepackage{subcaption}
\usepackage{tikz}
\usepackage{enumitem}

\newcommand{\tran}{{\mathrm{T}}}

\newcommand{\ifracB}[2]{{{#1}/{(#2)}}}

\newcommand{\ind}{\mathbbm{1}}

\makeatletter
\newcommand{\biggg}{\bBigg@{4}}
\newcommand{\Biggg}{\bBigg@{5}}
\makeatother

\newcommand{\circled}[1]{
    \tikz[baseline=(char.base)]{\node[shape=circle, draw, inner sep=0pt,minimum size=9pt] (char) {#1};}
}

\newcommand{\tcircled}[1]{
    \tikz[baseline=(char.base)]{\node[shape=circle, draw, inner sep=0pt,minimum size=6pt] (char) {\scriptsize #1};}
}

\newcounter{relctr} %
\everydisplay\expandafter{\the\everydisplay\setcounter{relctr}{0}} %

\newcommand\labelrel[2]{%
  \begingroup
    \refstepcounter{relctr}%
    \stackrel{\smash{\textnormal{(\alph{relctr})}}}{#1}%
    \originallabel{#2}%
  \endgroup
}

\AtBeginDocument{\let\originallabel\label} %
\crefformat{relctr}{(#2#1#3)}
\allowdisplaybreaks
\begin{document}

\title{Entrywise Error Bounds for Spectral Ranking with Semi-Random Adversaries}

\author{Dongmin Lee}
\authornote{The author ordering is alphabetical.}
\email{lee4818@purdue.edu}
\orcid{0009-0006-7962-2305}
\affiliation{%
  \department{Department of Computer Science}
  \institution{Purdue University}
  \city{West Lafayette}
  \state{IN}
  \country{USA}
}
\author{Anuran Makur}
\email{amakur@purdue.edu}
\orcid{0000-0002-2978-8116}
\affiliation{%
  \department{Department of Computer Science}
  \department{Elmore Family School of Electrical and Computer Engineering}
  \institution{Purdue University}
  \city{West Lafayette}
  \state{IN}
  \country{USA}
}
\author{Japneet Singh}
\email{sing1041@purdue.edu}
\orcid{0000-0002-4953-1465}
\affiliation{%
  \department{Elmore Family School of Electrical and Computer Engineering}
  \institution{Purdue University}
  \city{West Lafayette}
  \state{IN}
  \country{USA}
}

\renewcommand{\shortauthors}{Lee, Makur, and Singh}

\begin{abstract}
Bradley-Terry-Luce (BTL) model estimation is a well-established strategy to rank a collection of items given a dataset of pairwise comparisons. Although the theoretical performance of BTL estimation methods, such as spectral and maximum likelihood estimation, is well studied in the regime of uniformly sampled graphs, generalizing such results to a wider class of random graphs has proved challenging. In this work, we investigate the entry-wise error of spectral algorithms against a semi-random adversary that can arbitrarily boost the sampling probabilities of certain edges. We find that the performance of the unweighted spectral method is heavily dependent on the spectral properties of the generated graph. Furthermore, we show that asymptotic performance approaching that of uniformly sampled graphs can be recovered by appropriately reweighting the observed edges to counteract the adversary and restore the spectral gap. Finally, we provide numerical simulations that support our theoretical findings.
\end{abstract}


\maketitle

\section{Introduction}
\emph{Pairwise preference learning} is the classic problem of estimating preferences for a collection of $n$ items given a set of pairwise comparisons. The problem has many applications including sports rankings \cite{Elo1978}, recommendation systems \cite{BaltrunasMakcinskasRicci2010}, and search engine algorithms \cite{PageBrinMotwaniWinograd1999}. Recently, the problem has also attracted significant attention in the machine learning community as a practical approach to align large language model (LLM) outputs with human feedback (see, e.g., \cite{ChristianoLeikeBrown2017,OuyangWuJiang2022}). A popular model for this problem is the Bradley-Terry-Luce (BTL) model \cite{BradleyTerry1952,Luce1959,Plackett1975,Zermelo1929}, which assumes that each item $i$ is associated with a latent score $\alpha_i > 0$ such that the probability of item $i$ being preferred over item $j$ is $\ifracB{\alpha_i}{\alpha_i+\alpha_j}$.
There are two broad paradigms to estimate the score vector 
$\bm{\alpha}=(\alpha_1,\dots,\alpha_n)$ from partially observed pairwise comparison graphs: maximum likelihood estimation (MLE) \cite{Ford1957,Hunter2004,Shahetal2016} and spectral algorithms \cite{NegahbanOhShah2017, JadbabaieMakurShah2020,SpectralonGeneralMultiway,JadbabaieMakurShah2024}. Under well-behaved settings, such as Erd\H{o}s-R\'{e}nyi comparison graphs, it is known that both methods are optimal in terms of sample complexity \cite{Chenetal2019}.

However, assuming each edge to be independent and identically distributed might not be a realistic assumption, as real-life data is often clustered and non-uniform. For example, when ranking chess players, there might be more matches observed between players who live in the same country. Unfortunately, the best known analyses of BTL estimation on Erd\H{o}s-R\'{e}nyi graphs do not easily generalize to other graph models (see, e.g., \cite{Shahetal2016} which studies deterministic graphs), because they critically rely on the superior uniformity of such random graphs, such as the concentration of node degrees, in order to derive the tightest bound possible. In the regime of general deterministic graphs where such properties cannot be guaranteed, obtaining a tight bound that holds for all graphs is difficult. Indeed, known results in the general setting such as \cite{RankingGeneralGraphsLocality,LiShrotriyaRinaldo2022} fail to bound the $\ell^\infty$ error as tightly as possible for uniformly sampled graphs.

Thus, it is natural to seek a middle ground between the restrictiveness of uniformly sampled graphs and the intractability of arbitrary deterministic graphs. To this end, we consider the \emph{semi-random adversary} model, where an adversary can (with some restrictions) perturb an initially uniform sample. The semi-random adversary model has found widespread acceptance in many different areas such as graph theory (especially graph coloring and partitioning) \cite{BlumSpencer1995,FeigeKilian2001,MakarychevMakarychevVijayaraghavan2012,AsilisChenHansesn2026}, matrix completion \cite{ChengGe2018,KelnerLiLiu2023,GaoCheng2023}, and preconditioning \cite{JambulapatiLiMusco2021}. Recently, \cite{MonotoneAdversary} investigated the effectiveness of the MLE method for BTL estimation on semi-random graphs. However, to the best of our knowledge, such an analysis has not been established for the second major paradigm of BTL estimation\textemdash the spectral method. \emph{In this work, we address this direction by analyzing the theoretical performance of the spectral method against a semi-random adversary}.

In our semi-random adversary model, the adversary is allowed to choose a different sampling probability for each possible edge of the graph, as long as it is not smaller than some base probability $p$. This model encompasses important random graph models such as the stochastic block model (SBM) \cite{HollandBlackmondLeinhardt1983} that better represent real-world data than standard Erd\H{o}s-R\'{e}nyi models. Since a graph generated by such an adversary appears to contain no less information than an equivalent Erd\H{o}s-R\'{e}nyi graph with the same $p$, it may seem that the adversary's actions can only benefit us. However, this is not true; it is a well-known phenomenon that adding edges to a graph can paradoxically hurt the spectral properties that spectral methods rely on. For example, as \cref{fig:paradox} illustrates, adding an edge to a graph can, counter-intuitively, reduce the second smallest eigenvalue of its normalized Laplacian (also called the spectral gap) \cite{Fiedler1973}.
Indeed, \cite{EldanRaczSchramm2017} demonstrated that the addition of a random edge to an Erd\H{o}s-R\'{e}nyi graph decreases the spectral gap with positive probability.

Our theoretical analysis focuses on the entry-wise, or $\ell^\infty$, loss, as it is the most relevant metric for many important applications, such as top-$K$ ranking \cite{ChenSuh2015,Chenetal2019,MonotoneAdversary} and hypothesis testing \cite{LagrangianInference,MakurSingh2025,MakurSingh2025a}. We begin by studying the performance of the vanilla spectral method (rank centrality \cite{NegahbanOhShah2017}) against a semi-random adversary and find that the same entry-wise error bounds as those of uniform edge sampling can be obtained, provided that the spectral gap of the semi-randomly generated graph can be lower bounded by some constant with high probability. Furthermore, we highlight a certain class of stochastic block models (SBMs) as an example of a non-Erd\H{o}s-R\'{e}nyi random graph that satisfies the spectral gap condition.

In addition, to overcome the spectral gap requirement, we 
take inspiration from \cite{MonotoneAdversary} to 
introduce an edge-weighted variant of the spectral method. Intuitively, the purpose of the reweighting is to undo the adversary's efforts to weaken the graph's spectral gap. We argue that such a reweighting must be possible by invoking a monotone coupling argument, showing that there must exist a subgraph of a semi-randomly generated graph that recovers the spectral properties of an Erd\H{o}s-R\'{e}nyi graph, then noting that edge reweighting is the continuous relaxation of taking a subgraph. We first present a generic error bound of the weighted spectral method that holds for any weighting procedure. Next, under the reweighting scheme proposed in \cite{MonotoneAdversary}, we show that weighted spectral ranking of semi-random graphs can maintain error bounds consistent with Erd\H{o}s-R\'{e}nyi graphs when the graph is sufficiently dense. We conclude with numerical simulations to support our theoretical results.

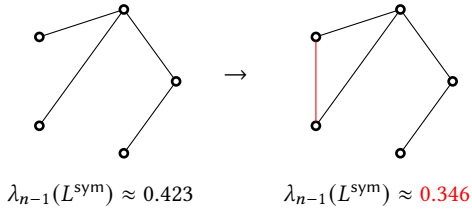
\begin{figure}[t]
    \centering
    
    \begin{tikzpicture}[baseline,every node/.style={circle,draw=black,line width=1pt,minimum size=3pt,inner sep=0,outer sep=0}]
        \node at (0:1) (a0) {};
        \node at (72:1) (a1) {};
        \node at (144:1) (a2) {};
        \node at (216:1) (a3) {};
        \node at (288:1) (a4) {};

        \draw (a0) -- (a4);
        \draw (a0) -- (a1);
        \draw (a1) -- (a2);
        \draw (a1) -- (a3);

        \node[rectangle,draw=none] at (270:1.5) (text) {$\lambda_{n-1}(L^{\mathsf{sym}})\approx 0.423$};
    \end{tikzpicture}\hspace{1em} $\to$ \hspace{1em}
    \begin{tikzpicture}[baseline,every node/.style={circle,draw=black,line width=1pt,minimum size=3pt,inner sep=0,outer sep=0}]
        \node at (0:1) (a0) {};
        \node at (72:1) (a1) {};
        \node at (144:1) (a2) {};
        \node at (216:1) (a3) {};
        \node at (288:1) (a4) {};

        \draw (a0) -- (a4);
        \draw (a0) -- (a1);
        \draw (a1) -- (a2);
        \draw (a1) -- (a3);
        \draw[red] (a2)--(a3);

        \node[rectangle,draw=none] at (270:1.5) (text) {$\lambda_{n-1}(L^{\mathsf{sym}})\approx \color{red}0.346$};
    \end{tikzpicture}
    \caption{A visualization of how adding an edge to a graph can paradoxically reduce its spectral gap, $\lambda_{n-1}(L^{\mathsf{sym}})$ (the second smallest eigenvalue of the normalized Laplacian).}
    \label{fig:paradox}
    \Description[Left: five-vertex graph, right: same graph with an added edge]{Left: a graph with five vertices and four edges: $(1, 2)$, $(1, 3)$, $(1, 4)$, and $(4, 5)$. Its spectral gap is $0.423$. Right: the same graph with the edge $(2, 3)$ added. Its spectral gap is now $0.346$.}
\end{figure}

\subsection{Contributions}
We summarize our contributions as follows.
\begin{itemize}[leftmargin=*]
    \item We obtain entry-wise error bounds for BTL estimation with the unweighted spectral method on semi-randomly generated graphs (\cref{thm:entry-wise-error-bounds-for-unweighted-rank-centrality}).
    \begin{itemize}
        \item Notably, we recover bounds consistent with \cite{Chenetal2019} when the spectral bound of the canonical Markov matrix is lower bounded by a nonzero constant with high probability, as is the case for Erd\H{o}s-R\'{e}nyi graphs (\cref{cor:error-bound-under-erdos-renyi-graphs}).
        \item Furthermore, we generalize this result to stochastic block models, showing that certain classes of SBMs preserve the spectral gap and thus allow error bounds comparable to those of Erd\H{o}s-R\'{e}nyi graphs to be attained (\cref{prop:error-bound-for-stochastic-block-model}).
    \end{itemize}
    \item We describe and analyze a variant of the spectral method that reweights edges to improve the spectral properties of the stochastic matrix in \cref{subsec:weighted-algorithm}.
    \begin{itemize}
        \item We first obtain error bounds for the generic case, where no assumptions are made about the graph or reweighting procedure (\cref{thm:error-bounds-for-weighted-rank-centrality}).
        \item Next, we derive error bounds for semi-randomly generated graphs using a specific reweighting procedure 
        (\cref{thm:error-bounds-for-weighted-rank-centrality-with-mmwu}).
    \end{itemize}
    \item We present empirical results that demonstrate how our weighted spectral method can counteract the effect of a semi-random adversary that chooses edge weights to degrade the spectral gap (\cref{sec:experiments}). However, we emphasize that our work is primarily theoretical in nature.
\end{itemize}
\subsection{Related Work}
The BTL model \cite{BradleyTerry1952,Luce1959,Plackett1975,Zermelo1929} is perhaps one of the most widely adopted framework to model outcomes of pairwise comparisons across a diverse range of domains, including sports \cite{barry1993choice,Massey1997,cattelan2013dynamic,JadbabaieMakurShah2020}, psychology \cite{Matthews2018pain}, animal behavior \cite{Adams2005}, and ranking scientific journals \cite{Stigler1994}. Recently, it has also found use in the machine learning community for tasks such as reinforcement learning from human feedback (RLHF) \cite{ChristianoLeikeBrown2017,OuyangWuJiang2022} and ranking LLMs \cite{Chiangetal2024}. Often the objective is to estimate the parameters of the BTL model (in order to find a ranking or determine the top $k$ items), and many techniques have been developed to tackle this problem, most notably MLE methods \cite{Ford1957,Hunter2004,Shahetal2016,Zermelo1929,SimonsYao1999} and spectral methods \cite{NegahbanOhShah2017, JadbabaieMakurShah2020,SpectralonGeneralMultiway} (and hybrids of the two such as \cite{ChenSuh2015}), but also least-squares \cite{HendrickxOlshevskySaligrama2020}, non-parametric \cite{BongLiShrotriyaRinaldo2020,Chatterjee2015}, and Bayesian \cite{GuiverSnelson2009,Adams2005,CaronDoucet2012} approaches.

There exists a large body of work that focuses on finding theoretical error bounds for these BTL estimation methods under certain graph configurations, such as round-robin tournaments (i.e., complete graphs) \cite{SimonsYao1999} and general fixed graphs \cite{NegahbanOhShah2017,LiShrotriyaRinaldo2022,SpectralonGeneralMultiway,RankingGeneralGraphsLocality}. In particular, the Erd\H{o}s-R\'{e}nyi model has attracted significant attention thanks to its simplicity and favorable statistical properties. For example, \cite{NegahbanOhShah2017} bounds the $\ell^2$ error of the spectral method on Erd\H{o}s-R\'{e}nyi random graphs and general graphs, while \cite{Chenetal2019,ChenGaoZhang2022} refine the former's results and also bound the $\ell^\infty$ errors of the spectral and MLE methods, finding both to be asymptotically optimal (though with differing leading constants). On the other hand,  \cite{UncertaintyQuantificationBTL} finds non-asymptotic entry-wise error bounds for the spectral and MLE methods under the Erd\H{o}s-R\'{e}nyi model and also presents some results for heterogeneously sampled graphs (which bears some similarities to our semi-random model, but unlike our model, requires edge sampling probabilities to have bounded dynamic range). Meanwhile, works such as \cite{SpectralonGeneralMultiway,UncertaintyQuantificationMultiway,JangKimSuh2018} analyze the error bounds of estimators for the Plackett-Luce (PL) model, which is a generalization of the BTL model for multi-way comparisons.  

\cite{MonotoneAdversary} introduced the semi-random adversary model in the context of BTL estimation and demonstrated that a weighted MLE method can achieve similar estimation guarantees to those of uniform sampling against a semi-random adversary. We remark that the use of reweighting as a technique to counteract a semi-random adversary has been studied in other contexts as well, such as matrix completion \cite{ChengGe2018} and linear regression \cite{JambulapatiLiMusco2021}.

\subsection{Notation}
We use the following notational conventions throughout this work. We represent vectors with lowercase bold letters and matrices with uppercase letters. Let $[n]=\{1,2,\dots,n\}$. $\|\cdot \|_p$ represents the $\ell^p$ norm for vectors and the corresponding induced norm for matrices. Given a vector $\bm{\pi}$, we define the vector norm $\|\bm{x}\|_{\bm{\pi}} \triangleq (\sum_{i=1}^n \pi_i x_i^2)^{1/2}$ and the corresponding left operator norm $\|A\|_{\bm{\pi}} \triangleq \sup_{\|\bm{x}\|_{\bm{\pi}}=1} \|\bm{x}^\tran A\|_{\bm{\pi}}$. Given a matrix $A$, $\lambda_i(A)$ refers to the $i$th largest eigenvalue of $A$ in absolute value. We use standard Bachmann-Landau asymptotic notation, e.g., $f(n)=O(g(n))$ if there exists some constant $M>0$ such that $|f(n)| \leq M g(n)$ for all $n \geq 1$. We use the term ``with high probability'' to describe events that occur with probability at least $1-O(n^{-t})$ for some $t\geq 1$.

\section{Problem Formulation}
\subsection{Bradley-Terry-Luce Model}
We consider a set of $n$ items, each labeled $i\in[n]$. Under the BTL model, we assume that each item $i$ is associated with an underlying score (or skill) parameter $\alpha_i>0$ such that the probability of item $j$ being preferred over $i$, $p_{ij}$, is given as
\begin{equation}
    \forall (i, j) \in [n]^2 \text{ such that } i \neq j,\enspace p_{ij} \triangleq \frac{\alpha_j}{\alpha_i+\alpha_j}.
\end{equation}
We assume that the \emph{dynamic range} of the scores is upper bounded by a constant $h$, i.e.,
\begin{equation}
    \frac{\max_{i \in [n]} \alpha_i}{\min_{i \in [n]} \alpha_i} \leq h. \label{eq:dynamic-range}
\end{equation}
Note that this implies
\begin{equation}
    \forall (i, j) \in [n]^2 \text{ such that } i\neq j,\enspace \frac{1}{1+h} \leq p_{ij} \leq \frac{h}{1+h}.
\end{equation}

\subsection{Semi-Random Observation Graph Model}
Let $\mathcal{G}=(\mathcal{V}, \mathcal{E})$ be the (undirected) observation graph that encodes observed comparisons, i.e., $(i, j)\in \mathcal{E}$ and $(j, i)\in \mathcal{E}$ if and only if items $i$ and $j$ are compared together. Given a parameter $p \in [0, 1]$, a \emph{semi-random adversary} generates the observation graph $\mathcal{G}$ by picking each edge $(i, j)$ with $i<j$ independently with probability $q_{ij}\in[p, 1]$. Note that if $p$ satisfies $np \geq c_0 \log(n)$ for some fixed $c_0 >1$, $\mathcal{G}$ is connected with high probability.
Let $d_\text{min}$ and $d_\text{max}$ be the minimum and maximum degree of the graph respectively.

Let $k$ be the number of observed comparisons for each pair $(i, j) \in \mathcal{E}$ with $i<j$. For simplicity, we assume that the number of comparisons is identical for every pair; our analysis can be extended to the case where the number of comparisons varies across pairs. The observed outcomes can be modeled as a sequence of independent Bernoulli random variables
\begin{equation}
Z_{ij}^{(m)} \sim \mathsf{Bernoulli}(p_{ij})
\end{equation} for all $m \in [k]$ and $(i, j)\in\mathcal{E}$ (with $i<j$), such that $\smash{Z_{ij}^{(m)}}=1$ if $j$ beat $i$ in the $m$-th comparison and $\smash{Z_{ij}^{(m)}}=0$ otherwise. Let
\begin{equation}
Z_{ij} \triangleq \sum_{m=1}^k Z_{ij}^{(m)}
\end{equation}
denote the total number of times $j$ beats $i$. Then,
\begin{equation}\hat{p}_{ij} \triangleq \frac{Z_{ij}}{k}\end{equation}
is the empirical probability of $j$ beating $i$.

Next, to provide intuition, we provide a simple example of a situation where the observation graph will follow the semi-random adversary model. Suppose that one compiles a database of pairwise comparisons between uniformly sampled items in some set $A$. In the middle of the data collection process, the scope of the experiment is expanded, and the set of items is updated to some superset of the original set, $A'\supset A$. As a result, every potential pair in $\{(i,j)\in {A'}^2: i<j\}$ has some base probability $p>0$ of being observed, but pairs in $A^2$ will have a higher probability of being observed and thus will be overrepresented in the dataset. Therefore, this dataset follows the semi-random observation graph model. As this example demonstrates, the semi-random adversary in our model is not necessarily an ``attacker'' but a theoretical construct that is useful to characterize the worst-case behavior of a wider class of random graphs that simpler models like the Erd\H{o}s-R\'{e}nyi model cannot capture.

\subsection{Spectral BTL Estimation}
The goal of BTL estimation algorithms is to estimate the underlying score vector $\bm{\alpha}$ given observed comparison outcomes. Since the score vector is scale-invariant (i.e., multiplying $\bm{\alpha}$ by any nonzero constant does not affect the induced probabilities $p_{ij}$), we normalize $\bm{\alpha}$ to obtain the canonical score vector to be estimated,
\begin{equation}
\bm{\pi} = \frac{1}{\sum_{i=1}^n \alpha_i} \begin{bmatrix}
    \alpha_1 & \cdots & \alpha_n
\end{bmatrix}.\label{eq:pi}
\end{equation}

Note that the dynamic range assumption in \cref{eq:dynamic-range} implies that
\begin{equation}
    \forall i \in [n],\enspace \frac{1}{nh} \leq \pi_i \leq \frac{h}{n},
\end{equation}
thus
\begin{alignat}{3}
    &&\frac{1}{h\sqrt{n}}\leq\frac{\sqrt{n}}{h} \|\bm{\pi}\|_\infty &\leq \|\bm{\pi}\|_2 &&\leq \sqrt{n} \|\bm{\pi}\|_\infty \leq \frac{h}{\sqrt{n}},\label{eq:pi-norm-relation}\\
\forall \bm{v} &\in \mathbb{R}^n,&\sqrt{\frac{1}{nh}} \|\bm{v}\|_2 &\leq \|\bm{v}\|_{\bm{\pi}} &&\leq \sqrt{\frac{h}{n}} \|\bm{v}\|_2,\\
    \forall A &\in \mathbb{R}^{n\times n},&\frac{1}{h} \|A\|_2 &\leq \|A\|_{\bm{\pi}} &&\leq h \|A\|_2.
\end{alignat}

Although there are several methods to estimate $\bm{\pi}$, the main focus of this work is the spectral method (also called rank centrality) \cite{NegahbanOhShah2017}.
Given $\mathcal{G}$ and the true probabilities $p_{ij}$, we can define the canonical Markov matrix $S \in [0, 1]^{n \times n}$ as
\begin{equation}
    S_{ij} \triangleq \begin{cases}
        \frac{p_{ij}}{d}, & \text{if } (i, j) \in \mathcal{E}, \\
        1 - \frac{1}{d} \sum_{l: (i, l) \in \mathcal{E}} p_{il}, & \text{if } i = j, \\
        0, & \text{otherwise,}
    \end{cases}\label{eq:S-definition}
\end{equation}
where $d=d_\text{max}$. Under the assumption that $\mathcal{G}$ is connected, the Markov chain is aperiodic (since $S_{ii} > 0$ for all $i\in[n]$), thus $S$ has a unique stationary distribution that can be shown to equal $\bm{\pi}$ \cite{LevinPeresWilmer2009}.

Motivated by this property, we can similarly define the empirical Markov matrix $\hat{S}$ using the empirical probabilities $\hat{p}_{ij}$ derived from the observations.
\begin{equation}
    \hat{S}_{ij} \triangleq \begin{cases}
        \frac{\hat{p}_{ij}}{d}, & \text{if } (i, j) \in \mathcal{E}, \\
        1 - \frac{1}{d} \sum_{l: (i, l) \in \mathcal{E}} \hat{p}_{il}, & \text{if } i = j, \\
        0, & \text{otherwise.}
    \end{cases}
\end{equation}
Since $\hat{S}$ will converge to $S$ as $k \to \infty$, the spectral method obtains an approximation of $\bm{\pi}$ (the normalized score vector) computes the stationary distribution $\hat{\bm{\pi}}$ of $\hat{S}$ (i.e., its principal eigenvector).

\section{Main Results on Spectral Method}

\subsection{Entry-Wise Error Bounds}
In this section, we establish entry-wise error bounds for the (unmodified) rank centrality algorithm when applied to an observation graph generated by a semi-random adversary. We show that the vanilla rank centrality algorithm (without modifications such as reweighting or trimming) maintains its entry-wise error bounds under a semi-random adversary as long as the spectral gap $1 - \lambda_2(S)$ of the canonical Markov matrix $S$ is sufficiently preserved, as the following theorem highlights.
\begin{theorem}[Entry-Wise Error Bounds for Unweighted Spectral Method]\label{thm:entry-wise-error-bounds-for-unweighted-rank-centrality}
Assume that the observation graph $\mathcal{G}$ is generated by a semi-random model with edge probabilities $q_{ij} \in [p, 1]$ for some constant $p \in (0, 1]$.
Let $S$ be the canonical Markov matrix of $\mathcal{G}$ as defined in \cref{eq:S-definition}.
Suppose that 
\begin{equation}
\forall i \in [n],\enspace n\sum_{j: j\neq i} q_{ij}^2 \leq s \left(\sum_{j:j\neq i} q_{ij}\right)^2 \label{eq:norm-ratio-bound-assumption}
\end{equation}
holds for some fixed constant $s > 1$, and that the spectral gap condition
\begin{equation}
    \mathbb{P}(1-|\lambda_2(S)|\leq \gamma) \leq \frac{1}{n^5} \label{eq:spectral-gap-assumption}
\end{equation}
holds for some constant $0<\gamma<1$.
Then, there exist constants $c_0\geq 10240h^2/\gamma^2$,
$c_1 > 0$, and $c_2>0$ such that for all $p,k$ satisfying $np \geq c_0 \log(n)$ and $k \geq 5$, the approximate score vector $\hat{\bm{\pi}}$ estimated by the spectral method satisfies
\begin{equation}
    \frac{\|\hat{\bm{\pi}}-\bm{\pi}\|_\infty}{\|\bm{\pi}\|_\infty} \leq \left(\frac{c_1}{\gamma}+c_2\right) \sqrt{\frac{\log(n)}{npk}} 
\end{equation}
for sufficiently large $n$ with probability at least $1-O(n^{-5})$.
\end{theorem}
\Cref{thm:entry-wise-error-bounds-for-unweighted-rank-centrality} is proved in \cref{app:unweighted-linf-proof}. A proof sketch is available in \cref{sec:proof-sketch-unweighted}. Some comments are due regarding each of the conditions that \cref{thm:entry-wise-error-bounds-for-unweighted-rank-centrality} requires, especially the variation condition \cref{eq:norm-ratio-bound-assumption} and spectral gap condition \cref{eq:spectral-gap-assumption}. First, \cref{eq:norm-ratio-bound-assumption} controls the asymptotic behavior of the $\ell^2$ norm of $\bm{q}_i=(q_{i1},\dots,q_{in})$ with respect to its $\ell^1$ norm, requiring that $\|\bm{q}_i\|_2/\|\bm{q}_i\|_1 = \Theta(1/\sqrt{n})$ for all $i\in[n]$. Alternatively, one may obtain a geometric interpretation by noting that the angle between $\bm{1}_n$ (the all-ones vector) and $\bm{q}_i$ is
\begin{equation}
\cos^{-1}\left(\frac{\bm{1}_n^\tran \bm{q}_i}{\|\bm{1}_n\|_2 \|\bm{q}_i\|_2}\right) = \cos^{-1}\left(\frac{\|\bm{q}_i\|_1}{\sqrt{n} \|\bm{q}_i\|_2}\right),
\end{equation}
thus \cref{eq:norm-ratio-bound-assumption} requires the angle between $\bm{1}_n$ and $\bm{q}_i$ to be upper bounded by $\cos^{-1}(1/\sqrt{s})$ for all $i\in[n]$.
We argue that this condition is not excessively restrictive, as it allows for several useful models such as stochastic block models (see proof of \cref{prop:error-bound-for-stochastic-block-model}). In fact, if $\bm{q}_i$ is uniformly sampled from $[0, 1]^n$, the quantity $\sqrt{n}\|\bm{q}_i\|_2/\|\bm{q}_i\|_1$ concentrates around $2/\sqrt{3}$ as $n\to\infty$. Thus, for sufficiently large $n$, most vectors $\bm{q}_i\in[0,1]^n$ satisfy \cref{eq:norm-ratio-bound-assumption} given $s > 4/3$.

The second condition, \cref{eq:spectral-gap-assumption}, requires the spectral gap of $S$, $1-|\lambda_2(S)|$ (recall that we order eigenvalues based on their absolute value in descending order), to be lower bounded by some constant with high probability. Such a requirement is expected because it is well established in the literature that BTL estimation critically depends on the spectral gap (or the effective resistance, a related quantity) of the graph (see, e.g., \cite{NegahbanOhShah2017,Shahetal2016,RankingGeneralGraphsLocality}). We emphasize, however, that the spectral gap is not the only piece in the puzzle; obtaining tight bounds as in \cref{thm:entry-wise-error-bounds-for-unweighted-rank-centrality} requires each term of the error to be carefully decomposed and controlled.

The key takeaway from this result is that when the spectral gap is lower bounded by a constant with high probability, the resulting error bounds are consistent with those for uniform sampling as derived in \cite{Chenetal2019}. Using the fact that Erd\H{o}s-R\'{e}nyi graphs preserve the spectral gap with high probability \cite{NegahbanOhShah2017}, we can recover the results of \cite{Chenetal2019} under uniform sampling in the following corollary.
\begin{corollary}[Error Bound for Erd\H{o}s-R\'{e}nyi Graphs]\label{cor:error-bound-under-erdos-renyi-graphs}
Suppose the observation graph $\mathcal{G}$ follows an Erd\H{o}s-R\'{e}nyi model with edge sampling probability $p$. Then, there exist constants $c_0>1$, $c_1>0$ such that for $np \geq c_0 \log (n)$, the approximate score vector $\hat{\bm{\pi}}$ estimated by the spectral method satisfies
\begin{equation}
    \frac{\|\hat{\bm{\pi}}-\bm{\pi}\|_\infty}{\|\bm{\pi}\|_\infty} \leq c_1 \sqrt{\frac{\log(n)}{npk}}
\end{equation}
with probability at least $1-O(n^{-5})$.
\end{corollary}

\subsection{Bounds for Stochastic Block Models}\label{subsec:sbm}
Next, we present a nontrivial class of stochastic block models (SBMs) that satisfy the spectral gap condition required by \cref{thm:entry-wise-error-bounds-for-unweighted-rank-centrality}. We begin with a formal definition of the SBM.
\begin{definition}[Stochastic Block Model {{\cite{HollandBlackmondLeinhardt1983,LoweTerveer2025}}}]
Let $m>1$ be the number of blocks, and suppose that the number of items $n$ is some integer multiple of $m$ such that $n/m\geq 2$. Without loss of generality, let $[n]$ be partitioned into $m$ blocks, $B_1=\{1,2,\dots,n/m\},B_2=\{n/m+1,\dots,2n/m\},\dots,\allowbreak B_m=\{n(m-1)/m+1,\dots,n\}$. Let $\bm{q}=(q_1,\dots,q_m)$ be the sampling probabilities for edges within each block, and let $p$ be the sampling probabilities for edges across blocks. Without loss of generality, assume that $q_1 \leq q_2 \leq \dots \leq q_m$, and suppose that $q_1>p$ (i.e., the SBM is assortative). Then, a graph $\mathcal{G}=([n],\mathcal{E})$ under the $\mathsf{SBM}(n,m,p,\bm{q})$ model is generated by independently picking each edge $(i, j)$ with probability $q_l$ if $i, j \in B_l$ for some $l \in [m]$ and $p$ otherwise.
\end{definition}

It is clear from the definition that the stochastic block model is a valid semi-random graph model with base probability $p$. Note that when studying the asymptotics of the SBM, we will assume that $m$ (the number of blocks) is constant with respect to $n$, but $p$ and $\bm{q}$ may depend on $n$. Using recent results by \cite{LoweTerveer2025} on the spectra of SBM graphs, we can derive a counterpart of \cref{cor:error-bound-under-erdos-renyi-graphs} for SBMs that satisfy a certain dynamic range condition.

\begin{proposition}[Error Bound for Stochastic Block Models]\label{prop:error-bound-for-stochastic-block-model}
Suppose the observation graph $\mathcal{G}\sim \mathsf{SBM}(n, m, p, \bm{q})$ with $np\geq c_0 \log^5(n)$ and $q_m \leq rp$ for some constants $c_0, r > 1$. Then, there exists some constant $c_1 > 0$ such that the approximate score vector $\hat{\bm{\pi}}$ estimated by the spectral method satisfies
\begin{equation}
    \frac{\|\hat{\bm{\pi}}-\bm{\pi}\|_\infty}{\|\bm{\pi}\|_\infty} \leq c_1 \sqrt{\frac{\log(n)}{npk}}
\end{equation}
with high probability.
\end{proposition}
\begin{proof}
Let $A$ be the adjacency matrix of $\mathcal{G}$ (i.e., $A_{ij} = \ind\{(i, j) \in \mathcal{E}\}$) and $D$ the degree matrix (i.e., $D=\text{diag}(\bm{d}=(d_1,\dots,d_n))$, where $d_i$ represents the degree of node $i$).

First, we must show that the variation bound \cref{eq:norm-ratio-bound-assumption} is satisfied.
For all $l \in [m]$ and $i \in B_l$,
\begin{align}
n\sum_{j:j\neq i} q_{ij}^2 &= \left(\frac{n^2}{m} - n\right) q_l^2 + \frac{n^2(m-1)}{m} p^2 \\
&\leq 2m \left(\left(\frac{n}{m}-1\right)^2 q_l^2 + \frac{n^2(m-1)^2}{m^2} p^2\right) \\
&\leq 2m \left(\sum_{j:j\neq i} q_{ij}\right)^2,
\end{align}
thus \cref{eq:norm-ratio-bound-assumption} holds with $s=2m$ (recall that we assume $m$ to be a constant with respect to $n$).

Next, we must control the spectral gap of $S$. From \cite[Corollary 4.3]{LoweTerveer2025}, we have that
\begin{equation}
1-|\lambda_2(A^{\mathsf{rw}})| \geq \gamma
\end{equation}
with high probability for some constant $\gamma>0$, where $A^{\mathsf{rw}} \triangleq D^{-1} A$ refers to the simple random walk matrix of $\mathcal{G}$, which is similar to the symmetric normalized adjacency matrix $D^{-1/2}AD^{-1/2}$. Note that conditions (1.1) and (1.2) of \cite{LoweTerveer2025} are always satisfied and the assortativity constant $\kappa$ defined in \cite[Equation (1.9)]{LoweTerveer2025} is always positive due to our assumptions on $p$ and $\bm{q}$.

Since the stationary distribution of $A^{\mathsf{rw}}$ is known to be $\bm{d}/\sum_i d_i$, we can invoke a comparison theorem \cite[Lemma 6]{NegahbanOhShah2017} to relate the spectral gaps of $S$ and $A^{\mathsf{rw}}$. With high probability,
\begin{align}
1-|\lambda_2(S)| &\geq  \frac{\displaystyle \min_{(i,j)\in\mathcal{E}} \frac{\pi_i S_{ij}}{d_i A^{\mathsf{rw}}_{ij}}}{\displaystyle \max_i \frac{\pi_i}{d_i}} (1-|\lambda_2(A^{\mathsf{rw}})|)\\
&\geq \frac{d_{\text{min}}}{d_{\text{max}}}\frac{\displaystyle\min_{(i,j)\in\mathcal{E}} \frac{\pi_i \pi_j}{\pi_i+\pi_j}}{\|\bm{\pi}\|_\infty} \gamma \\
&\geq \frac{d_{\text{min}}}{d_{\text{max}}} \frac{1}{2h^2} \gamma,
\end{align}
which is a positive constant. Thus, all conditions of \cref{thm:entry-wise-error-bounds-for-unweighted-rank-centrality} are met, completing the proof.
\end{proof}
Note that the density requirement $np \geq c_0 \log^5(n)$ is only needed so that the results of \cite{LoweTerveer2025} can be applied to our regime, and is not necessarily a limitation of the algorithm. It is possible that \cref{prop:error-bound-for-stochastic-block-model} can be generalized to a broader class of SBMs, e.g., a regime where $m$ increases as a function of $n$, but we do not consider such cases for the sake of simplicity.

\section{Main Results on Weighted Spectral Method}
Although our results in the previous section demonstrate that the standard spectral method remains effective in a broader range of settings beyond the Erd\H{o}s-R\'{e}nyi model, such as the important stochastic block model, no guarantees on performance can be made on graphs with vanishing spectral gaps. For example, an SBM that violates the dynamic range condition of \cref{prop:error-bound-for-stochastic-block-model} may have an asymptotically vanishing spectral gap (see, e.g., \cref{sec:experiments}). In this section, we address this limitation and provide theoretical guarantees that are robust against graphs with undesirable spectral properties. To this end, we examine a version of rank centrality with \emph{edge reweighting}, taking inspiration from a line of work including \cite{ChengGe2018,JambulapatiLiMusco2021,MonotoneAdversary} that studies how semi-random observations can be reweighted to more closely resemble a uniform sample which enjoys better guarantees.
\subsection{Weighted Algorithm Overview}\label{subsec:weighted-algorithm}
The goal of the weighted rank centrality algorithm is to assign weights $w_{ij} \in [0, 1]$ to each edge of the observation graph such that the weighted graph $\mathcal{G}=([n], \mathcal{E}, \{w_{ij}\})$ has better spectral properties. The weights are symmetric (i.e., $w_{ji}=w_{ij}$), and $w_{ij}=0$ for $(i, j) \notin \mathcal{E}$. Let
\begin{align}
    d_\text{max} &\triangleq \max_{i \in [n]} \sum_{j: (i, j) \in \mathcal{E}} w_{ij}, \\
    d_\text{min} &\triangleq \min_{i \in [n]} \sum_{j: (i, j) \in \mathcal{E}} w_{ij},
\end{align}
denote the maximum and minimum weighted degree of $\mathcal{G}$ respectively, and let $w_\text{max}\leq 1$ be the maximum edge weight.

Next, we redefine the Markov matrix $S$ and its empirical counterpart $\hat{S}$ to incorporate these weights.
\begin{align}
    S_{ij} &\triangleq \begin{cases}
        \frac{p_{ij} w_{ij}}{d}, & \text{if } (i, j) \in \mathcal{E}, \\
        1 - \frac{1}{d} \sum_{l: (i, l) \in \mathcal{E}} p_{il} w_{il}, & \text{if } i = j, \\
        0, & \text{otherwise,}
    \end{cases}\label{eq:weighted-S-definition} \\
    \hat{S}_{ij} &\triangleq \begin{cases}
        \frac{\hat{p}_{ij} w_{ij}}{d}, & \text{if } (i, j) \in \mathcal{E}, \\
        1 - \frac{1}{d} \sum_{l: (i, l) \in \mathcal{E}} \hat{p}_{il} w_{il}, & \text{if } i = j, \\
        0, & \text{otherwise,}
    \end{cases}\label{eq:weighted-S-hat-definition}
\end{align}
where $d=d_\text{max}$.
Observe that the stationary distribution $\pi$ of $S$ remains unchanged and is given by \cref{eq:pi}. Similarly to the original rank centrality algorithm, the weighted rank centrality algorithm estimates the normalized score vector by computing the stationary distribution $\hat{\pi}$ of the weighted empirical Markov matrix $\hat{S}$.

Critical to the performance of this algorithm is the procedure in which the weights $w_{ij}$ are chosen. As a tool to analyze the effectiveness of the chosen weights, we define the weighted Laplacian matrix $L^W$ of $\mathcal{G}$,
\begin{equation}
    L_w \triangleq \sum_{(i, j) \in \mathcal{G}, i<j} w_{ij} (\bm{e}_i-\bm{e}_j)(\bm{e}_i-\bm{e}_j)^\tran,
\end{equation}
where $\bm{e}_i$ is the standard $i$th basis vector. The spectral gap of $L^W$,
\begin{equation}
    \lambda_{n-1}(L^W) = \min_{\bm{v} \in \mathbb{R}^n: \bm{v}^\tran \bm{1}_n=0} \frac{\bm{v}^\tran L^W \bm{v}}{\|\bm{v}\|_2^2}, 
\end{equation}
the second smallest eigenvalue of $L^W$ (also called the Fiedler value or algebraic connectivity), characterizes the spectral properties of the weighted graph. Note that $\lambda_{n-1}(L^W)$ only depends on the graph $\mathcal{G}$ and the set of weights $w_{ij}$, not the underlying score vector $\bm{\pi}$ or probabilities $p_{ij}$.

\subsection{Entry-Wise Error Bounds}
In this section, we present the $\ell^\infty$-error bounds for the weighted rank centrality algorithm. We begin by analyzing the general case for any arbitrary set of weights, then discuss a method to effectively choose these weights in order to guarantee good performance against a semi-random adversary. 

First, we derive the following error bound for a deterministic observation graph that holds for any given set of weights.
\begin{theorem}[Error Bounds for Generic Weighted Rank Centrality]\label{thm:error-bounds-for-weighted-rank-centrality}
Suppose that the weighted observation graph $\mathcal{G}=([n], \mathcal{E}, \{w_{ij}\})$ is connected. Assume that $d_\text{min} \geq 1$, $d_\text{max} \leq 2np$ for some $p$, and $k \geq 640 h^2 n \log(n) p / \lambda_{n-1}(L^W)^2$.

Then, there exists constants $c_1,c_2,c_3>0$ such that the approximate score vector $\hat{\bm{\pi}}$ generated by the weighted spectral method satisfies
\begin{align}
    \frac{\|\hat{\bm{\pi}}-\bm{\pi}\|_\infty}{\|\bm{\pi}\|_\infty} \leq \Biggl(&\frac{c_1}{\lambda_{n-1}(L^W)} \sqrt{\frac{n \log(n) p}{k}} \\
    &+\frac{c_2}{\lambda_{n-1}(L^W)^2}\sqrt{\frac{n^2\log(n) p^2}{k}} \\
    &+\frac{c_3}{\lambda_{n-1}(L^W)^3}\sqrt{\frac{n^4\log(n) p^3}{k}} \Biggr)
\end{align}
with probability at least $1-O(n^{-5})$.
\end{theorem}
\Cref{thm:error-bounds-for-weighted-rank-centrality} is proved in \cref{app:proof-error-bounds-for-weighted-rank-centrality}. A proof sketch is provided in \cref{sec:proof-sketch-weighted}. Note that the error bounds depend almost entirely on the spectral properties of $\mathcal{G}$, namely $\lambda_{n-1}(L^W)$. In particular, one can recover a result similar to \cref{cor:error-bound-under-erdos-renyi-graphs} from this theorem by observing that unweighted Erd\H{o}s-R\'{e}nyi graphs satisfy $\lambda_{n-1}(L^W) \geq np/2$ (and the $d_\text{min}$ and $d_\text{max}$ constraints) with high probability (see \cite[Lemma 10]{MonotoneAdversary}). Thus, as long as we can find an appropriate set of weights such that $\lambda_{n-1}(L^W)$ grows like $np$ (i.e., so that $\mathcal{G}$ ``behaves like'' an Erd\H{o}s-R\'{e}nyi graph spectrally), we can achieve ``Erd\H{o}s-R\'{e}nyi-like'' performance for a larger class of graphs, such as semi-randomly generated graphs.

To show that such a set of weights must exist, we can utilize a monotone coupling argument (see \cref{lem:recovering-properties-of-erdos-renyi-graphs}) to show that any property that holds with high probability for Erd\H{o}s-R\'{e}nyi graphs can also be recovered with high probability from a well-chosen subgraph of a semi-randomly generated graph. Furthermore, given any subgraph $\mathcal{G}'=([n], \mathcal{E}')$ of $\mathcal{G}$, we can choose weights
\begin{equation}
    w_{ij} = \begin{cases}
        1, & (i, j) \in \mathcal{E'}, \\
        0, & \text{otherwise}
    \end{cases}
\end{equation}
so that the weighted graph is equivalent to $\mathcal{G}'$. This implies that there must exist a set of weights such that the weighted graph satisfies $\lambda_{n-1}(L^W)\geq np/2$ with high probability. Thus, the problem reduces to finding the weights (under the $d_\text{min}$ and $d_\text{max}$ constraints) that maximize $\lambda_{n-1}(L^W)$.

As \cite{MonotoneAdversary} showed, this optimization problem can be formulated as a semi-definite program (SDP) approximately solvable in near-linear time with the Matrix Multiplicative Weight Update (MMWU) algorithm \cite[Algorithm 2]{MonotoneAdversary}. Since the MMWU method can $1/2$-approximate the optimization problem in near-linear time, we can obtain the following guarantee for the MMWU-weighted spectral method.
\begin{theorem}[Error Bounds for MMWU-Weighted Spectral Method]\label{thm:error-bounds-for-weighted-rank-centrality-with-mmwu}
Assume that the observation graph is generated by a semi-random adversary with base probability $p$ satisfying $np \geq c_0 \log(n)$ for some constant $c_0>1$. Then, there exists some constants $c_1,c_2,c_3>0$ such that for all
$k \geq 10240h^2$, the approximate score vector $\hat{\bm{\pi}}$ estimated by the weighted spectral method with MMWU \cite[Algorithm 2]{MonotoneAdversary} satisfies
\begin{equation}
\frac{\|\hat{\bm{\pi}}-\bm{\pi}\|_\infty}{\|\bm{\pi}\|_\infty} \leq c_1 \sqrt{\frac{\log(n)}{npk}} + c_2 \sqrt{\frac{n\log(n)}{(np)^3k}} 
\end{equation}
with probability at least $1-O(n^{-5})$.
\end{theorem}

\Cref{thm:error-bounds-for-weighted-rank-centrality-with-mmwu} follows from \cref{thm:error-bounds-for-weighted-rank-centrality} and \cite[Theorem 10]{MonotoneAdversary}. Note that the bound in \cref{thm:error-bounds-for-weighted-rank-centrality-with-mmwu} can be simplified when the graph is sufficiently dense, as the following corollary demonstrates.

\begin{corollary}[Error Bounds for MMWU-Weighted Spectral Method on Dense Graphs]\label{cor:error-bounds-mmwu-dense}
In addition to the assumptions of \cref{thm:error-bounds-for-weighted-rank-centrality-with-mmwu}, suppose that $np \geq c\sqrt{n}$ for some constant $c>0$. Then, the error bound simplifies to
\begin{equation}
\frac{\|\hat{\bm{\pi}}-\bm{\pi}\|_\infty}{\|\bm{\pi}\|_\infty} \leq C \sqrt{\frac{\log(n)}{npk}} 
\end{equation}
with high probability for some constant $C>0$.
\end{corollary}

Thus, the MMWU-weighted spectral method is asymptotically optimal for sufficiently dense graphs. Note that it is still possible for the spectral gap to decay towards zero even in such dense regimes. We emphasize that this theoretical result does not imply that the weighted spectral method is unviable for sparse semi-random graphs, but only that optimal bounds in such settings is yet unavailable. It is possible that alternate proof techniques may allow optimal bounds to be obtained for sparser regimes. We believe this to be a promising direction for future work.

\section{Analysis and Proof Sketches for Main Results}
In this section, we present brief proof sketches for each of our main results and state some necessary lemmas.
\subsection{Proof Sketch of \cref{thm:entry-wise-error-bounds-for-unweighted-rank-centrality}}\label{sec:proof-sketch-unweighted}
Before the entry-wise error of the unweighted spectral method can be bounded, a tight bound on the $\ell^2$ error is needed.
\begin{lemma}[$\ell^2$ Error Bounds for Unweighted Spectral Method {{\cite[Theorem 5.2]{Chenetal2019}}}]\label{lem:l2-error-bounds-for-unweighted-rank-centraliity}
Under the setting of \cref{thm:entry-wise-error-bounds-for-unweighted-rank-centrality}, with probability at least $1-O(n^{-10})$, we have
\begin{equation}
    \frac{\|\hat{\bm{\pi}}-\bm{\pi}\|_2}{\|\bm{\pi}\|_2} \leq \frac{c_3}{\gamma\sqrt{npk}}
\end{equation}
for some constant $c_3$.
\end{lemma}
\Cref{lem:l2-error-bounds-for-unweighted-rank-centraliity} is proved in \cref{app:proof-l2-unweighted}. Note that simply using \cref{lem:l2-error-bounds-for-unweighted-rank-centraliity} in combination with the norm inequality $\|\cdot\|_\infty \leq \|\cdot\|_2 \leq \sqrt{n}\|\cdot\|_\infty$ is not sufficient to obtain the tight bound of \cref{thm:entry-wise-error-bounds-for-unweighted-rank-centrality}.

In order to obtain a tight bound, we decompose the entry-wise error to multiple components and make extensive use of concentration inequalities such as Hoeffding's inequality and Bernstein's inequality \cite{BoucheronLugosiBousquet2003} to bound the deviation of each component from the mean with high probability. A key component is the celebrated leave-one-out argument introduced by \cite{Chenetal2019}, which isolates the randomness of a single node to simplify analysis. Although the general flavor of our proof follows \cite{Chenetal2019}, the heterogeneity of our semi-random setting presents unique challenges, as our edges are not identically distributed.

At a high level, our proof decomposes the error term into
\begin{align}
\hat{\pi}_i-\pi_i &= \frac{1}{\boxed{\textstyle \sum_j \hat{p}_{ij}}} \biggl( \boxed{\textstyle\sum_j (\hat{p}_{ji}\pi_j-\hat{p}_{ij} \pi_i)} + \boxed{\textstyle \sum_j \hat{p}_{ji} (\hat{\pi}_j - \hat{\pi}^{(-i)}_j)} \\
&\quad + \boxed{\textstyle\sum_j \hat{p}_{ji} (\hat{\pi}^{(-i)}_j-\bar{\pi}^{(-i)}_j)} + \boxed{\textstyle\sum_j \hat{p}_{ji} (\bar{\pi}^{(-i)}_j - \pi_j)} \biggr),\label{eq:sketch-decomposition}
\end{align}
where the leave-one-out vector $\hat{\bm{\pi}}^{(-i)}$ is defined as the stationary distribution of the leave-one-out stochastic matrix
\begin{equation}
    (\hat{S}^{(-m)})_{ij} \triangleq \begin{cases}
        \hat{S}_{ij}=\frac{\hat{p}_{ij}}{d}, & \text{if } i \neq m \text{ and } j \neq m \text{ and } i\neq j \\
        \frac{p_{ij} q_{ij}}{d}, & \text{if (} i=m \text{ or } j=m \text{) and } i \neq j,\\
        1-\sum_{l:l\neq i} (\hat{S}^{(-m)})_{il}, & \text{if } i=j,
    \end{cases}
\end{equation}
and $\bar{\bm{\pi}}^{(-i)}$ is defined as
\begin{align}
    \bar{\pi}^{(-m)}_i &\triangleq \frac{\sum_{j\notin\{i,m\}} \hat{p}_{ji} \pi_j + p_{mi} q_{mi} \pi_m}{\sum_{j\notin\{i,m\}} \hat{p}_{ij} + p_{im} q_{im}}.
\end{align}
Each of the boxed terms in \cref{eq:sketch-decomposition} are bounded separately, allowing each source of error to be isolated. Finally, the terms, each bounded with high probability, are combined with a union bound to obtain the final bound with the desired probability.

\subsection{Proof Sketch of \cref{thm:error-bounds-for-weighted-rank-centrality}}\label{sec:proof-sketch-weighted}
Similarly to the proof of \cref{thm:entry-wise-error-bounds-for-unweighted-rank-centrality}, the entry-wise error is decomposed into multiple components that are bounded separately. However, the flavor of analysis is significantly different because the graph is fixed. In particular, we make use of various results on graph Laplacian spectra such as \cite[Lemma 18]{MonotoneAdversary} and matrix concentration inequalities \cite{Tropp2015} to characterize the relationship between error terms and Laplacian eigenvalues. In addition, we use ideas from, e.g., \cite{SpectralonGeneralMultiway} to recursively bound the $\ell^2$ and $\ell^\infty$ norms of various error terms to obtain a tight bound.
\subsection{Analysis of \cref{thm:error-bounds-for-weighted-rank-centrality-with-mmwu}}
In order to obtain results such as \cref{thm:error-bounds-for-weighted-rank-centrality-with-mmwu}, it is essential to formally characterize the relationship between an Erd\H{o}s-R\'{e}nyi graph and a semi-randomly generated graph with the same $p$ by defining a coupling between the two random variables.
\begin{definition}[Monotone Coupling Between Erd\H{o}s-R\'{e}nyi Graphs and Semi-Random Graphs]\label{def:maximal-coupling}
Given $n\in\mathbb{N}$, $p \in (0, 1]$, and $\{q_{ij} \in [p, 1] \mid \forall i,j\in\mathbb{N}: i<j\leq n\}$, for each pair $(i, j) \in [n]^2$ such that $i < j$, draw an independent uniformly distributed random variable $U_{ij} \sim \textsf{Unif}(0, 1)$. Construct two graphs, $\mathcal{G}_\text{E}=([n], \mathcal{E}_\text{E})$ and $\mathcal{G}_\text{S}=([n], \mathcal{E}_\text{S})$, such that $(i, j) \in \mathcal{E}_\text{E}$ if and only if $U_{ij} \leq p$ and $(i, j) \in \mathcal{E}_\text{S}$ if and only if $U_{ij} \leq q_{ij}$. The graphs $\mathcal{G}_\text{E}$ and $\mathcal{G}_\text{S}$ are realizations of the \emph{monotone coupling} between the random variables $G_\text{E}$, which is distributed according to the Erd\H{o}s-R\'{e}nyi model with edge sampling probability $p$, and $G_\text{S}$, which is distributed according to the semi-random model with edge sampling probabilities $\{q_{ij}\}$. Furthermore, $\mathbb{P}(G_\text{E} \subseteq G_\text{S})=1$. 
\end{definition}
Under this coupling, the edge set of the semi-randomly generated graph $\mathcal{G}_\text{S}$ is always a superset of the edge set of $\mathcal{G}_\text{E}$. This implies that given a realization of a semi-randomly generated graph $\mathcal{G}_\text{S}$, it is possible to probabilistically choose an Erd\H{o}s-R\'{e}nyi subgraph of $\mathcal{G}$ according to $\mathbb{P}(G_\text{E} \mid G_\text{S} = \mathcal{G}_\text{S})$. Critically, this allows us to show that any useful property of Erd\H{o}s-R\'{e}nyi graphs can be recovered by a carefully chosen subgraph of a semi-randomly generated graph, as the following lemma shows.
\begin{lemma}[Recovering Properties of Erd\H{o}s-R\'{e}nyi graphs from Subgraphs of Semi-Random Graphs]\label{lem:recovering-properties-of-erdos-renyi-graphs}
Let $\mathcal{P}$ be some graph property that holds for an Erd\H{o}s-R\'{e}nyi graph (with edge sampling probability $p$) with high probability, i.e., 
\begin{equation}
\mathbb{P}_{G_\text{E} \sim \mathsf{ER}(n,p)}(G_\text{E} \in \mathcal{P}) \geq 1 - O(n^{-t})
\end{equation}
for some fixed $t>0$.
Let $G_\text{S}$ be a semi-randomly generated graph under edge sampling probabilities $\{q_{ij}\}$. Then, with high probability, there must exist a subgraph of $G_\text{S}$ that satisfies $\mathcal{P}$, i.e.,
\begin{equation}
\mathbb{P}_{G_\text{S} \sim \mathsf{SR}(n,\{q_{ij}\})}(\exists \mathcal{G} \subseteq G_\text{S}: \mathcal{G} \in \mathcal{P}) \geq 1 - O(n^{-t}).
\end{equation}
\end{lemma}

\begin{proof}
Given a graph $\mathcal{G}_\text{S}$, let $G_\text{E}$ be distributed according to $\mathbb{P}(G_\text{E}\mid G_\text{S}=\mathcal{G}_\text{S})$ under the coupling defined in \cref{def:maximal-coupling}. Note that $\mathbb{P}(G_\text{E} \subseteq \mathcal{G}_\text{S})=1$ by definition. Then,
\begin{equation}
\mathbb{P}_{G_\text{E}\mid G_\text{S}}(G_\text{E} \in \mathcal{P} \mid G_\text{S}=\mathcal{G}_\text{S}) \leq \ind\{\exists \mathcal{G} \subseteq \mathcal{G}_\text{S}: \mathcal{G} \in \mathcal{P}\}.
\end{equation}
Thus,
\begin{align}
&\phantom{{}={}}\mathbb{P}_{G_\text{S} \sim \mathsf{SR}(n,\{q_{ij}\})}(\exists \mathcal{G} \subseteq G_\text{S}: \mathcal{G} \in \mathcal{P})
\\ &= \mathbb{E}_{G_\text{S} \sim \mathsf{SR}(n,\{q_{ij}\})}[\ind\{\exists \mathcal{G} \subseteq G_\text{S}: \mathcal{G} \in \mathcal{P}\}] \\
&\geq \mathbb{E}_{G_\text{S} \sim \mathsf{SR}(n,\{q_{ij}\})}[\mathbb{P}_{G_\text{E}|G_\text{S}}(G_\text{E} \in \mathcal{P} \mid G_\text{S})] \\
&= \mathbb{P}_{G_\text{E} \sim \mathsf{ER}(n,p)}(G_\text{E} \in \mathcal{P}) \geq 1 - O(n^{-t}),
\end{align}
which completes the proof.
\end{proof}
Although this lemma establishes the existence of an appropriate subgraph with high probability, it does not provide an explicit method to recover one. Indeed, it is not trivial to recover an Erd\H{o}s-R\'{e}nyi subgraph from a semi-randomly generated graph, especially if the underlying edge sampling probabilities are unknown. Thankfully, the mere existence of a valid subgraph is sufficient to obtain results such as \cref{thm:error-bounds-for-weighted-rank-centrality-with-mmwu}.
\begin{figure*}[t]
    \centering
    \begin{subfigure}[t]{0.3\textwidth}
        \includegraphics[width=\linewidth]{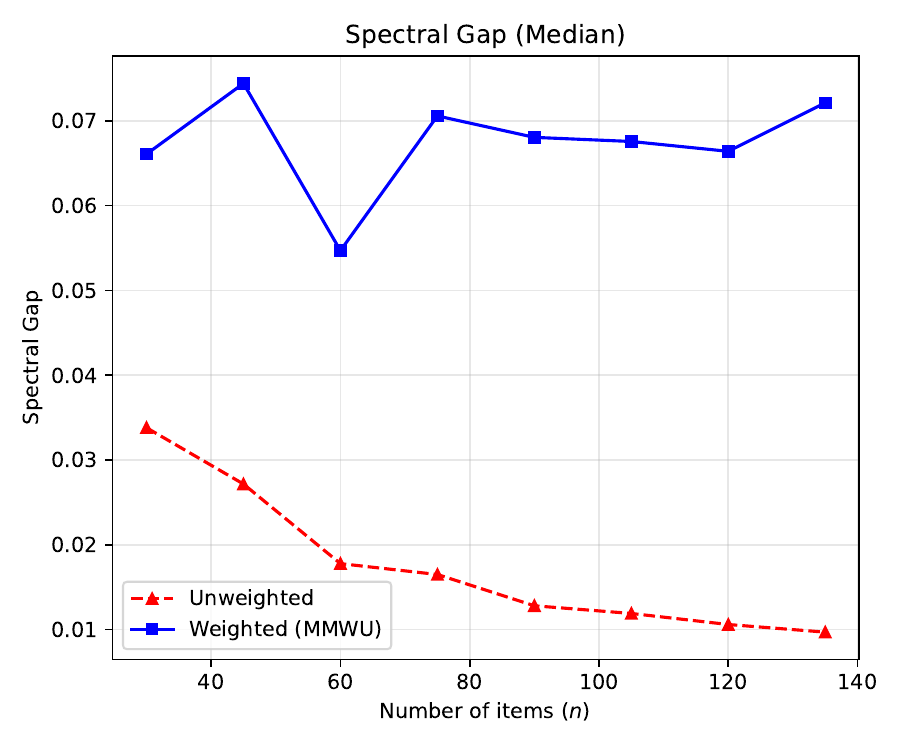}
        \caption{Spectral gaps ($1-|\lambda_2(S)|$) of the unweighted and MMWU-weighted canonical Markov matrices as a function of $n$ (Experiment 1).}
    \end{subfigure}
    \hspace{1em}
    \begin{subfigure}[t]{0.3\textwidth}
        \includegraphics[width=\linewidth]{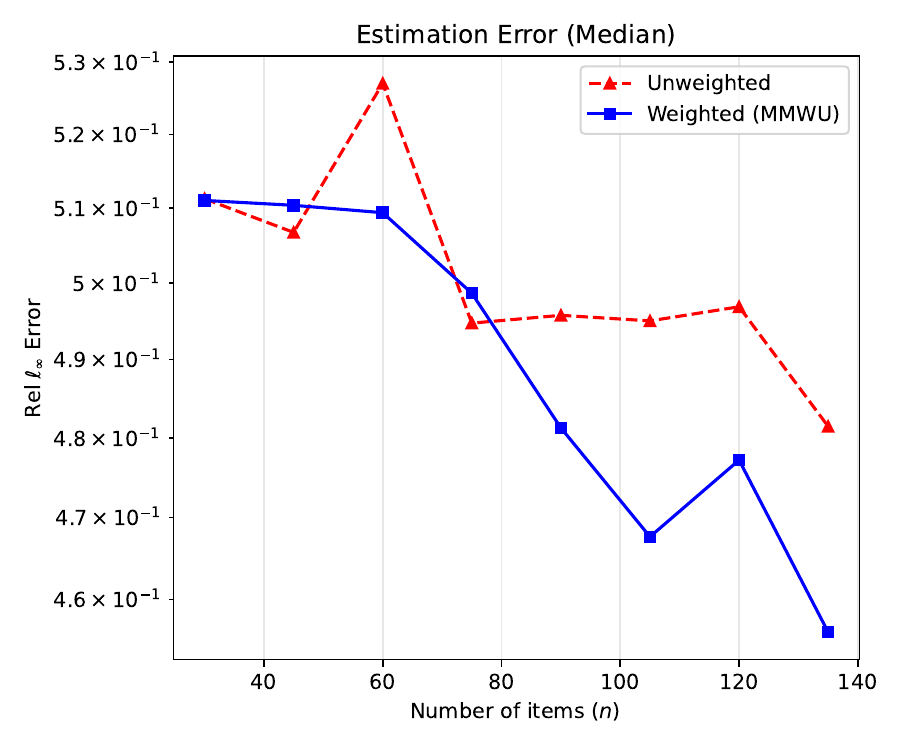}
        \caption{Relative $\ell^\infty$-error ($\frac{\|\hat{\bm{\pi}}-\bm{\pi}\|_\infty}{\|\bm{\pi}\|_\infty}$) of the unweighted and MMWU-weighted spectral method with respect to $n$ (Experiment 1).}
    \end{subfigure}
    \hspace{1em}
    \begin{subfigure}[t]{0.3\textwidth}
        \includegraphics[width=\linewidth]{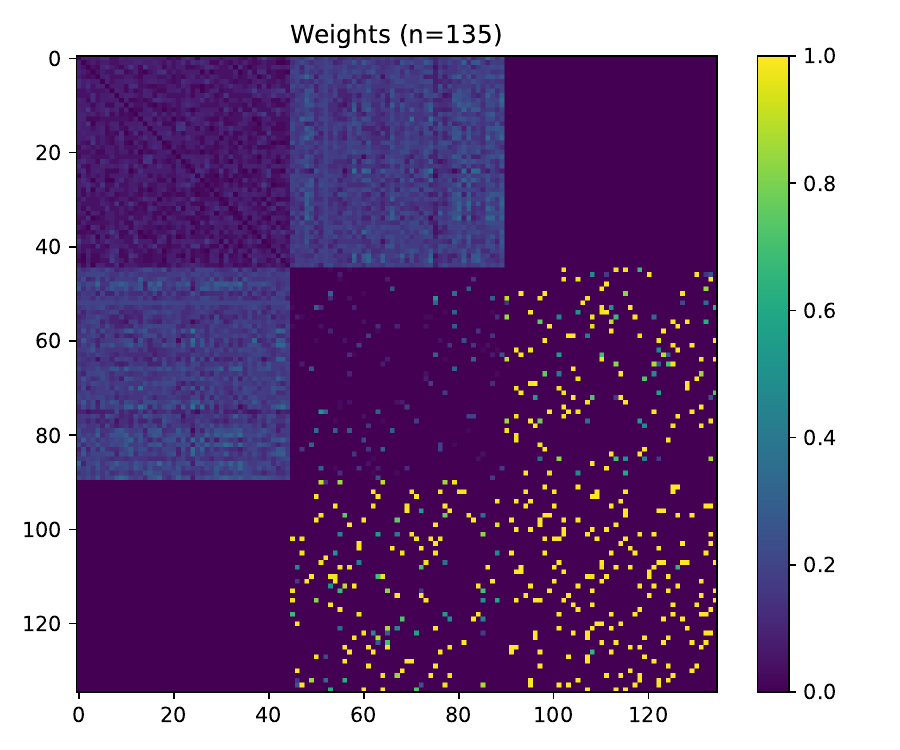}
        \caption{Visualization of the MMWU-computed weights in Experiment 1 ($n=135$).}
    \end{subfigure}
    \begin{subfigure}[t]{0.3\textwidth}
        \includegraphics[width=\linewidth]{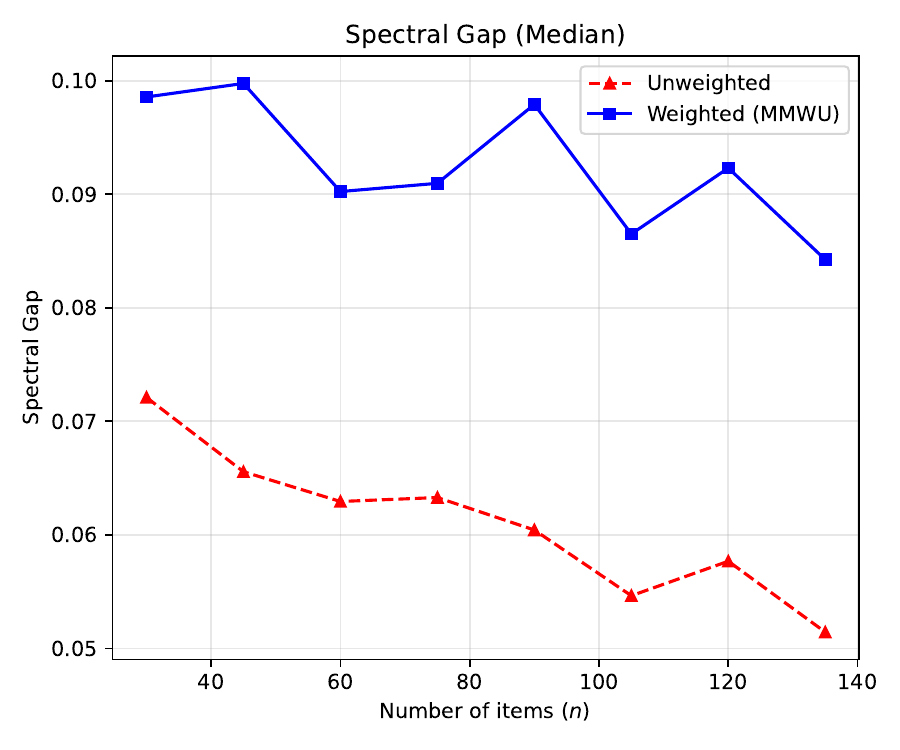}
        \caption{Spectral gaps ($1-|\lambda_2(S)|$) of the unweighted and MMWU-weighted canonical Markov matrices as a function of $n$ (Experiment 2).}
    \end{subfigure}
    \hspace{1em}
    \begin{subfigure}[t]{0.3\textwidth}
        \includegraphics[width=\linewidth]{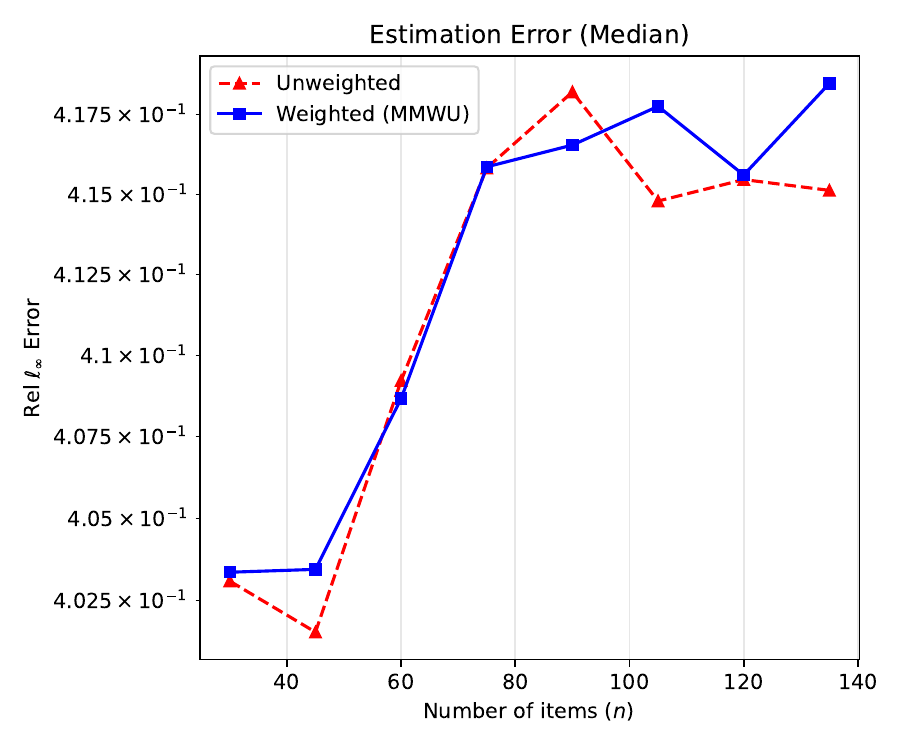}
        \caption{Relative $\ell^\infty$-error ($\frac{\|\hat{\bm{\pi}}-\bm{\pi}\|_\infty}{\|\bm{\pi}\|_\infty}$) of the unweighted and MMWU-weighted spectral method with respect to $n$ (Experiment 2).}
    \end{subfigure}
    \hspace{1em}
    \begin{subfigure}[t]{0.3\textwidth}
        \includegraphics[width=\linewidth]{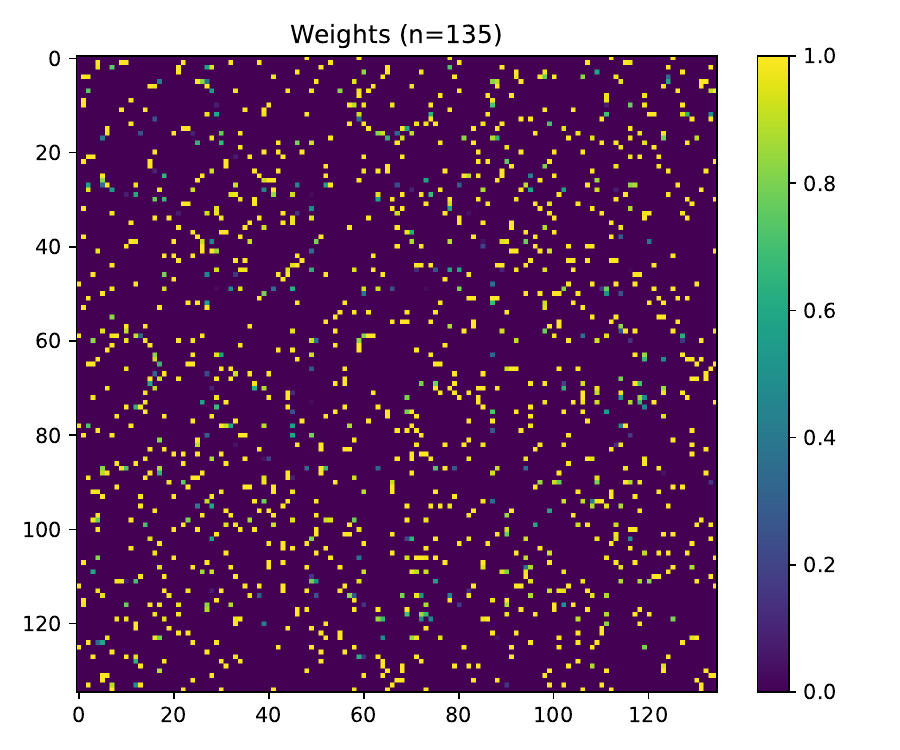}
        \caption{Visualization of the MMWU-computed weights in Experiment 2 ($n=135$).}
    \end{subfigure}
    \caption{\ul{Experiment 1 (Subfigures (a)--(c)):} $3$-block SBM with edge sampling probabilities between each pair of blocks given by the matrix \cref{eq:sbm-matrix}. \ul{Experiment 2 (Subfigures (d)--(f)):} Erd\H{o}s-R\'{e}nyi model with edge sampling probability $p = 2\log(n)/n$.}
    \label{fig:experiments}
    \Description[(a)--(c): plots on SBM dataset. (d)--(f): plots on Erd\H{o}s-R\'{e}nyi model dataset.]{(a) Line plot that compares the spectral gaps of the unweighted and weighted canonical Markov matrices on SBM data distributed according to \cref{eq:sbm-matrix}. The $x$ axis is the number of items ($n$) and the $y$ axis is the spectral gap. The spectral gap of the unweighted matrix is around $0.03$ when $n=30$ and decreases to around $0.01$ when $n=135$, while the spectral gap of the weighted matrix stays relatively constant at around $0.07$ regardless of $n$. (b) Line plot that compares the $\ell^\infty$ estimation error of the unweighted and weighted spectral methods on SBM data. The $x$ axis is the number of items ($n$) and the $y$ axis is the relative $\ell^\infty$ error. The relative estimation error decreases as the number of items increases for both the unweighted and weighted methods, but the decrease is more pronounced for the weighted method. When $n=30$, both methods have a similar relative error of around $0.51$, but when $n=135$, the relative error of the weighted method drops to around $0.46$, whereas it is around $0.48$ for the unweighted method. (c) Heatmap of the MMWU weights assigned to each edge of the SBM graph. The heatmap consists of multiple colored dots, where each dot at $(x, y)$ corresponds to the edge between items $x$ and $y$, and the color represents the weight, which is between $0$ and $1$. Dots near the lower right, which correspond to the sparser part of the SBM, are brighter (i.e., weighted closer to $1$), while dots near the upper left, which correspond to denser parts of the SBM, are darker (i.e., weighted closer to $0$). (d) Line plot that compares the spectral gaps of the unweighted and weighted canonical Markov matrices on an Erd\H{o}s-R\'{e}nyi dataset. Both plots tend downwards, but there is a clear gap between them. The spectral gap of the unweighted algorithm decreases from around $0.07$ to $0.05$ as $n$ increases from $30$ to $135$, while the spectral gap of the weighted algorithm changes from around $0.1$ to $0.08$. (e) Line plot that compares the $\ell^\infty$ estimation error of the unweighted and weighted spectral methods on Erd\H{o}s-R\'{e}nyi data. The two methods behave very similarly, and it is not clear which one is more effective. They both exhibit an error of about $0.4$ when $n=30$, which jumps to around $0.415$ as $n$ increases to $90$ and stays relatively constant beyond that point. (f) Heatmap of the MMWU weights assigned to each edge of the Erd\H{o}s-R\'{e}nyi graph. The assigned weights are relatively uniform, and no clear pattern is visible.}
\end{figure*}

\section{Experiments}\label{sec:experiments}
In this section, we conduct two numerical experiments to complement our theoretical results. The purpose of the experiments is to highlight a setting where reweighting is beneficial, and another where it is not. In each experiment, we compare the performance of the unweighted and MMWU-weighted spectral methods on random graphs of various sizes. To reduce the effect of randomness, we run each experiment $25$ times and take the median. The algorithms were implemented on Python, and the greedy $1/2$-approximation oracle described in \cite[Theorem 5]{MonotoneAdversary} was used to implement the MMWU algorithm.

\textbf{Experiment 1.} In the first experiment (\cref{fig:experiments}(a)--(c)), graphs are generated from a stochastic block model. Although we only considered SBMs with a single inter-block probability in our theoretical exposition in \cref{subsec:sbm}, we will use the more general definition of a stochastic block model here, where the edge probabilities between pairs of blocks are given as an arbitrary symmetric matrix. The $3$-block SBM used in this experiment has block edge sampling probabilities given by the matrix

\begin{equation}
    P=\begin{bmatrix}
        1 & 1 & 0 \\
        1 & 2 \frac{\log(n)}{n} & 2 \frac{\log(n)}{n} \\
        0 & 2 \frac{\log(n)}{n} & 2 \frac{\log(n)}{n}
    \end{bmatrix},\label{eq:sbm-matrix}
\end{equation}
where $n\in\{30, 45, 60, 75, 90, 105, 120, 135\}$ is the number of items (note that we choose multiples of $3$ to ensure that each block has the same size). Note that this SBM is a valid semi-random model (albeit with $p=0$), but does not satisfy the assumptions of our main theoretical results. Of course, as the simulation results show, this does not preclude the MMWU-weighted spectral method from being effective.

The setting of Experiment 1 has several properties that make reweighting especially worthwhile. For example, the node degrees exhibit large variation. While nodes in the first block each have degree $\Theta(n)$, the average degree of nodes in the third block is only $\Theta(\log(n))$. Thus, the ratio $d_{\text{max}}/d_{\text{min}}$ is not bounded as $n$ increases. In such settings, reweighting is beneficial because it effectively normalizes node degrees, improving the spectral properties of the stochastic matrix significantly. In addition, as \cref{fig:experiments}(a) shows, the spectral gap of the unweighted canonical Markov matrix decays to zero as $n$ increases in this model. MMWU reweighting counteracts this decrease to keep the spectral gap relatively constant with respect to $n$.

\Cref{fig:experiments}(b) plots the relative entry-wise error of the unweighted and weighted spectral methods on this model as a function of $n$. As the plot shows, the MMWU-weighted spectral method outperforms the standard spectral method, especially as $n$ increases. \Cref{fig:experiments}(c), a heatmap of the edge weights assigned by the MMWU algorithm, gives us a glimpse into the inner workings of the reweighting process. Regions that are overrepresented (the top left corner) are assigned smaller weights, while samples that are underrepresented are given higher precedence.

\textbf{Experiment 2.} On the other hand, the second experiment (\cref{fig:experiments}(d)--(f)) serves as an example of a regime where reweighting is relatively unnecessary. In this experiment, the graph is generated from an Erd\H{o}s-R\'{e}nyi model with edge sampling probability $p=2\log(n)/n$ (with $n$ taking the same values as before). Since Erd\H{o}s-R\'{e}nyi graphs are known to have good spectral properties, the benefit of reweighting is relatively minor.

Interestingly, as \cref{fig:experiments}(d) shows, the MMWU method is still effective in increasing the spectral gap of the stochastic matrix, even though the unweighted spectral gap is already well controlled. However, this increase in spectral gap does not translate to improved performance, as \cref{fig:experiments}(e) illustrates. Since the purpose of the reweighting is to recover Erd\H{o}s-R\'{e}nyi-like properties, it is expected that it is not of much use if the initial graph is already generated from an Erd\H{o}s-R\'{e}nyi model. Furthermore, \cref{fig:experiments}(f) demonstrates, the weights chosen by the MMWU algorithm in this setting are mostly evenly distributed, unlike \cref{fig:experiments}(c). Still, even in this ``worst case'' setting, the weighted spectral method still holds its ground against the unweighted version.

\begin{acks}
This work is supported in part by the \grantsponsor{ccf-2337808}{National Science Foundation (NSF)}{https://www.nsf.gov/} CAREER Award under Grant \grantnum{ccf-2337808}{CCF-2337808}.
\end{acks}

\bibliographystyle{ACM-Reference-Format}
\bibliography{refs}

\appendix

\newcommand{\whprel}[2]{\stackrel[\mathclap{#1}]{}{\boxed{#2}}}
\newcommand{\awhprel}[2]{\!\!\whprel{#1}{#2}}

\newcommand\labelawhprel[3]{%
  \begingroup
    \refstepcounter{relctr}%
    \!\!\stackrel[\mathclap{#1}]{\smash{\textnormal{(\alph{relctr})}}}{\boxed{#2}}%
    \originallabel{#3}%
  \endgroup
}
\section{Introduction to Appendices}
In the appendices, we prove the various theorems and propositions introduced in the body. We use the shorthand
\begin{equation}
a \whprel{1-O(n^{-t})}{\geq} b
\end{equation}
to say ``$a \geq b$ with probability at least $1-O(n^{-t})$''.

\section{Proofs on Unweighted Spectral Method}\label{app:proof-entry-wise-error-bounds-for-unweighted}
\subsection[Proof of Lemma 5.1]{Proof of \cref{lem:l2-error-bounds-for-unweighted-rank-centraliity}}\label{app:proof-l2-unweighted}
First, we prove \cref{lem:l2-error-bounds-for-unweighted-rank-centraliity}.
\begin{proof}
We begin by invoking \cite[Theorem 5.1]{Chenetal2019} to get
\begin{equation}
    \|\bm{\pi}-\hat{\bm{\pi}}\|_2 \leq \sqrt{nh} \|\bm{\pi}-\hat{\bm{\pi}}\|_{\bm{\pi}} \leq  \frac{\sqrt{nh}\ \boxed{\|\bm{\pi}^\tran (S-\hat{S})\|_{\bm{\pi}}}\mathrlap{\ \circled{a}}}{\boxed{1-|\lambda_2(S)|-\|S-\hat{S}\|_{\bm{\pi}}}\mathrlap{\ \circled{b}}}.
\end{equation}
First, we lower bound \circled{b}. From \cite[Lemma 3]{NegahbanOhShah2017}, we have
\begin{equation}
\|S-\hat{S}\|_{\bm{\pi}} \leq h \|S-\hat{S}\|_2 \whprel{1-O(n^{-10})}{\leq} 80h \sqrt{\frac{\log(n)}{k d_{\text{max}}}}.
\end{equation}
Furthermore, we can use a Chernoff bound to show that $d_{\text{max}} \geq np/2 \geq c_0 \log(n)/2$ with probability at least $1-O(n^{-10})$ (see proof of \cref{thm:entry-wise-error-bounds-for-unweighted-rank-centrality} below). Thus,
\begin{equation}
    \circled{b} \whprel{1-O(n^{-10})}{\geq} \gamma - 80h \sqrt{\frac{\log(n)}{k d_{\text{max}}}} \quad \whprel{1-O(n^{-10})}{\geq}\  \gamma-80h\sqrt{\frac{2}{c_0 k}} \geq \frac{\gamma}{2}.
\end{equation}
Next, we upper bound \circled{a}. Following the proof of \cite[Theorem 5.2]{Chenetal2019}, we have
\begin{equation}
\circled{a} \leq \sqrt{\frac{h}{n}} \|\bm{\pi}^\tran (S-\hat{S})\|_2 \whprel{1-O(n^{-10})}{\lesssim}\ \sqrt{\frac{1}{n}} \frac{\|\bm{\pi}\|_2}{\sqrt{npk}}.
\end{equation}
Combining the bounds, we have
\begin{equation}
\frac{\|\bm{\pi}-\hat{\bm{\pi}}\|_2}{\|\bm{\pi}\|_2}\quad \whprel{1-O(n^{-10})}{\lesssim}\quad \frac{1}{\gamma\sqrt{np k}},
\end{equation}
completing the proof.
\end{proof}
\subsection[Proof of Theorem 3.1]{Proof of \cref{thm:entry-wise-error-bounds-for-unweighted-rank-centrality}} \label{app:unweighted-linf-proof}
Now we are ready to prove \cref{thm:entry-wise-error-bounds-for-unweighted-rank-centrality}.
\begin{proof}
Before we begin, we will first condition on the event that $1-\lambda_2(S)\geq \gamma$, which by assumption has probability at least $1-n^{-5}$.

First, let $A$ be the adjacency matrix of the observed $\mathcal{G}$. In particular, this implies that
\begin{equation}
    \mathbb{E}[A_{ij}] = q_{ij}
\end{equation}
and
\begin{equation}
    \mathbb{E}[\hat{p}_{ij} \mid A_{ij} = a] = p_{ij} a,\enspace \forall a \in \{0, 1\}.
\end{equation}
We will use the shorthands $\mathbb{P}(\cdot \mid A)$ and $\mathbb{E}[\cdot \mid A]$ to refer to probabilities and expectations conditioned on $A$. We also define
\begin{equation}
    d_i \triangleq \sum_{j} A_{ij}
\end{equation}
to be the degree of node $i$. Since each $A_{ij}\sim \textsf{Bernoulli}(q_{ij})$, using a Chernoff bound, we have
\begin{align}
\mathbb{P}\left(\left|d_i-\sum_j q_{ij}\right| \geq \frac{1}{2} \sum_j q_{ij}\right) &\leq 2\exp\left(-\frac{\sum_j q_{ij}}{12}\right) \\
&\leq 2 \exp\left(-\frac{(n-1)p}{12}\right) \\
&\leq 3 \exp\left(-\frac{c_0}{12} \log(n)\right) \\&\leq 3n^{-10} \label{eq:di-bounds}
\end{align}
since $c_0\geq 120$. In particular, since all $q_{ij}\geq p$, we have $d_i \geq np/2 \geq c_0 \log(n)/2$ with high probability.
Next, recall that since $\hat{\bm{\pi}}$ is the stationary distribution of $\hat{S}$, we have
\begin{equation}
\hat{\pi}_i = \frac{\sum_j \hat{p}_{ji} \hat{\pi}_j}{\sum_j \hat{p}_{ij}}.
\end{equation}
Note that $\mathbb{E}[\hat{p}_{ji}]=p_{ji} q_{ji}$.

We introduce the auxiliary vector $\bar{\bm{\pi}}$ as
\begin{equation}
\bar{\pi}_i = \frac{\sum_j \hat{p}_{ji} \pi_j}{\sum_j \hat{p}_{ij}}.
\end{equation}

Next, we decompose the entry-wise error as
\begin{align}
\overset{\mathclap{\circled{x}=\|\hat{\bm{\pi}}-\bm{\pi}\|_\infty}}{\boxed{\hat{\pi}_i-\pi_i}} &= \overset{\triangleq \delta_i}{\boxed{\hat{\pi}_i-\bar{\pi}_i}} + \bar{\pi}_i-\pi_i \label{eq:unweighted-delta-def}\\
&= \delta_i + \frac{\boxed{\textstyle\sum_j (\hat{p}_{ji}\pi_j-\hat{p}_{ij} \pi_i)}\mathrlap{\ \circled{a}}}{\boxed{\textstyle \sum_j \hat{p}_{ij}}\mathrlap{\ \circled{b}}}.
\end{align}

First, we bound \circled{a} with high probability.
\begin{align}
    &\phantom{{}={}} \mathbb{P}\left(|\circled{a}| \leq \|\bm{\pi}\|_\infty \sqrt{\frac{20d_i\log(n)}{k}}\ \middle|\ A \right) \\
    &\geq \left(1-2\exp\left(-\frac{2\left(\|\bm{\pi}\|_\infty \sqrt{\frac{20 d_i \log(n)}{k}}\right)^2}{\sum_j A_{ij}(\pi_i+\pi_j)^2 / k}\right)\right) \\
    &\geq 1-2n^{-10}.
\end{align}

Next, we lower bound \circled{b}.
\begin{align}
    &\phantom{{}={}} \mathbb{P}\left(\circled{b} \geq \frac{d_i}{2(1+h)}\ \middle|\ A \right) \\
    &\geq \mathbb{P}\left(k\ \circled{b} \geq \frac{k}{2} \sum_j A_{ij} p_{ij}\ \middle|\ A \right) \\
    &\geq 1-\exp\left(-\frac{k \sum_j A_{ij} p_{ij}}{8}\right) \\
    &\geq 1 - \exp\left(-\frac{kd_i}{8(1+h)}\right)\\
    &\geq 1 - \exp\left(-\frac{knp}{16(1+h)}\right) \\
    &= 1 - \exp\left(-\frac{c_0k \log(n)}{16(1+h)}\right) \geq 1-n^{-10}
\end{align}
for sufficiently large $c_0$.

Thus,
\begin{align}
\left|\frac{\circled{a}}{\circled{b}}\right| \quad\  &\awhprel{1-O(n^{-10})}{\leq}\  \frac{\|\bm{\pi}\|_\infty \sqrt{\frac{20d_i\log(n)}{k}}}{ \frac{d_i}{2(1+h)}}\\
&=2(1+h)\|\bm{\pi}\|_\infty \sqrt{\frac{20\log(n)}{d_i k}} \\
&\awhprel{1-O(n^{-10})}{\leq} 2(1+h)\|\bm{\pi}\|_\infty \sqrt{\frac{40\log(n)}{npk}}.
\end{align}

Next, we analyze $\delta_i$, which was defined at \cref{eq:unweighted-delta-def}. To do so, we introduce $\hat{S}^{(-m)}$, a leave-one-out version of $\hat{S}$ with the $m$th row and column replaced by their expected values, i.e.,
\begin{equation}
    (\hat{S}^{(-m)})_{ij} \triangleq \begin{cases}
        \hat{S}_{ij}=\frac{\hat{p}_{ij}}{d}, & \text{if } i \neq m \text{ and } j \neq m \text{ and } i\neq j \\
        \frac{p_{ij} q_{ij}}{d}, & \text{if (} i=m \text{ or } j=m \text{) and } i \neq j,\\
        1-\sum_{l:l\neq i} (\hat{S}^{(-m)})_{il}, & \text{if } i=j.
    \end{cases}
\end{equation}
Let $\hat{\bm{\pi}}^{(-m)}$ be the stationary distribution of $\hat{S}^{(-m)}$. Similarly, let $\bar{\bm{\pi}}^{(-m)}$ be the leave-one-version of $\bar{\bm{\pi}}$, i.e.,
\begin{equation}
    \bar{\pi}^{(-m)}_i \triangleq \frac{\sum_{j\notin\{i,m\}} \hat{p}_{ji} \pi_j + p_{mi} q_{mi} \pi_m}{\sum_{j\notin\{i,m\}} \hat{p}_{ij} + p_{im} q_{im}}.
\end{equation}
Finally, $d^{(-m)}_i$ is the leave-one-out counterpart of $\delta_i$,
\begin{equation}
    \delta^{(-m)}_i \triangleq \hat{\pi}^{(-m)}_i-\bar{\pi}^{(-m)}_i.
\end{equation}
Then, we can decompose $\delta_i$ as
\begin{align}
\delta_i = \frac{1}{\underset{\circled{b}}{\boxed{\textstyle \sum_j \hat{p}_{ij}}}} \biggl (&\underset{\circled{1}}{\boxed{\textstyle \sum_j \hat{p}_{ji} (\hat{\pi}_j - \hat{\pi}^{(-i)}_j)}} + \underset{\circled{2}}{\boxed{\textstyle\sum_j \hat{p}_{ji} \delta^{(-i)}_j}} \\
&+ \underset{\circled{3}}{\boxed{\textstyle\sum_j \hat{p}_{ji} (\bar{\pi}^{(-i)}_j - \pi_j)}} \biggr).
\end{align}

First, we bound \circled{1}.
\begin{equation}
    \circled{1} \leq \sqrt{\sum_j \hat{p}^2_{ji}} \|\hat{\bm{\pi}}-\hat{\bm{\pi}}^{(-i)}\|_2 \leq \sqrt{d_i}\ \underset{\circled{c}}{\boxed{\|\hat{\bm{\pi}}-\hat{\bm{\pi}}^{(-i)}\|_2}}.
\end{equation}

Next, we focus on \circled{3}.
\begin{equation}
\circled{3} = \sum_j \hat{p}_{ji} \ \frac{\boxed{\textstyle \sum_{l\notin\{i,j\}} \hat{p}_{lj} \pi_l -\hat{p}_{jl} \pi_j}\mathrlap{\ \circled{a}'}}{\boxed{\textstyle\sum_{l\notin\{i,j\}} \hat{p}_{jl} + p_{ji} q_{ji}}\mathrlap{\ \circled{b}'}}.
\end{equation}

Observe that $\circled{a}'$ and $\circled{b}'$ are leave-one-out versions of \circled{a} and \circled{b} respectively. Thus, we can similarly conclude that
\begin{equation}
|\circled{a}'| \whprel{1-O(n^{-10})}{\leq} \|\bm{\pi}\|_\infty \sqrt{\frac{20(d_j-1)\log(n)}{k}}
\end{equation}
and
\begin{equation}
\circled{b}'\  \whprel{1-O(n^{-10})}{\geq}\quad \frac{d_j-1}{2(1+h)}.
\end{equation}

Thus,
\begin{align}
\circled{3} &= \sum_j \left(\hat{p}_{ji} \frac{\circled{a}'}{\circled{b}'}\right) \leq \overset{\circled{b}}{\boxed{\textstyle\sum_j \hat{p}_{ji}}} \max_j \frac{|\circled{a}'|}{\circled{b}'} \\
&\awhprel{1-O(n^{-9})}{\leq} 2 \ \circled{b} \ (1+h) \|\bm{\pi}\|_\infty \sqrt{\frac{20\log(n)}{(d_i-1)k}} \\
&\awhprel{1-O(n^{-10})}{\leq} 2 \ \circled{b} \ (1+h) \|\bm{\pi}\|_\infty \sqrt{\frac{80\log(n)}{npk}},
\end{align}
assuming that $np\geq 4$ (which will hold for sufficiently large $n$).

Now, we focus on \circled{2}. To bound \circled{2} with high probability, our strategy will be to first bound the expectation of \circled{2}, then show that with high probability, the realization of \circled{2} cannot deviate from the expectation too much. The expectation will be taken with respect to the connectivity of node $i$ only. Thus, it will be conditioned on $\mathcal{G}^{(-i)}$, the graph $\mathcal{G}$ without node $i$. We also condition on the comparison data between $i$ and every other node $j \neq i$, $\tilde{p}_{ji}$. Then, we have that $(\hat{p_{ji}}\mid \mathcal{G}^{(-i)}, \tilde{\bm{p}}_{i}) \sim \textsf{Bernoulli}(q_{ji})$.

Under this setting, let us first analyze $\mathbb{E}[\circled{2} \mid \mathcal{G}^{(-i)}, \tilde{\bm{p}}_{i}]$.
\begin{align}
&\phantom{{}={}}\mathbb{E}[\circled{2} \mid \mathcal{G}^{(-i)}, \tilde{\bm{p}}_{i}] \\&= \sum_j \tilde{p}_{ji} q_{ji} \delta_j^{(-i)} \\
&\leq \underset{\circled{d}}{\boxed{\sqrt{\sum_j q_{ji}^2}}}\  \|\hat{\bm{\pi}}^{(-i)}-\bar{\bm{\pi}}^{(-i)}\|_2 \\
&\leq \circled{d} \left(\underset{\circled{c}}{\boxed{\|\hat{\bm{\pi}}^{(-i)}-\hat{\bm{\pi}}\|_2}}+\underset{\cref{lem:l2-error-bounds-for-unweighted-rank-centraliity}}{\boxed{\|\hat{\bm{\pi}}-\bm{\pi}\|_2}}+\|\bm{\pi}-\bar{\bm{\pi}}^{(-i)}\|_2\right) \\
&\awhprel{1-O(n^{-10})}{\leq}\  \circled{d} \left(\circled{c}+\frac{c_3 \sqrt{n} \|\bm{\pi}\|_\infty}{\gamma \sqrt{np k}}+\sqrt{n}\  \underset{\max \frac{\left|\tcircled{a}\right|'}{\tcircled{b}'}}{\boxed{\|\bm{\pi}-\bar{\bm{\pi}}^{(-i)}\|_\infty}}\right) \\
&\awhprel{1-O(n^{-9})}{\leq}\  \circled{d} \left(\circled{c}+ \frac{\|\bm{\pi}\|_\infty}{\sqrt{pk}}\left(\frac{c_3}{\gamma}+8\sqrt{5}(1+h)\sqrt{\log(n)}\right) \right).\label{eq:unweighted-exp2}
\end{align}
Next, \cref{eq:norm-ratio-bound-assumption} implies that
\begin{equation}
\circled{d} \leq s \frac{\sum_j q_{ji}}{\sqrt{n}}\quad \whprel{1-O(n^{-10})}{\leq}\quad \frac{2sd_i}{\sqrt{n}}.
\end{equation}

Next, we bound the deviation of \circled{2} from the expectation with high probability.
\begin{align}
&\phantom{{}={}}\mathbb{P}\Biggl(\left|\circled{2} - \mathbb{E}[\circled{2} \mid \mathcal{G}^{(-i)}, \tilde{\bm{p}}_{i}]\right| \geq\\
&\qquad \underset{\circled{e}}{\boxed{\|\bm{\delta}^{(-i)}\|_\infty\left(2\sqrt{5}\sqrt{\sum_j q_{ji}(1-q_{ji}) \log(n)} + \frac{20}{3}\log(n)\right)}}\ \Biggr) \\
&=\mathbb{P}\left(\left|\sum_j \tilde{p}_{ji} (A_{ji}-q_{ji}) \delta_j^{(-i)}\right| \geq \circled{e}\right) \\
&\leq 2\exp\left(-\frac{\frac{1}{2}\circled{e}^2}{\sum_j q_{ji}(1-q_{ji}) \left(\tilde{p}_{ji} \delta_j^{(-i)} \right)^2+\frac{1}{3}\|\bm{\delta}^{(-i)}\|_\infty\circled{e}}\right) \\
&\leq 2\exp\left(-\frac{\frac{1}{2}\circled{e}^2}{\left(\sum_j q_{ji}(1-q_{ji})\right) \|\bm{\delta}^{(-i)}\|_\infty^2+\frac{1}{3}\|\bm{\delta}^{(-i)}\|_\infty\circled{e}}\right)\\
&= 2\exp\left(-\frac{\frac{1}{2}\left(2\sqrt{5}\sqrt{\sum_j q_{ji}(1-q_{ji}) \log(n)} + \frac{20}{3}\log(n)\right)^2}{\splitfrac{\sum_j q_{ji}(1-q_{ji})}{ +\frac{1}{3}\left(2\sqrt{5}\sqrt{\sum_j q_{ji}(1-q_{ji}) \log(n)} + \frac{20}{3}\log(n)\right)}}\right) \\
&\leq 2\exp\left(-10\log(n)\right)=2 n^{-10}.
\end{align}
Since
\begin{equation}
\sqrt{\sum_j q_{ji}(1-q_{ji})} \leq \sqrt{\sum_j q_{ji}}\ \whprel{1-O(n^{-10})}{\leq}\  \sqrt{2d_i},
\end{equation}
we have
\begin{align}
&\left|\circled{2} - \mathbb{E}[\circled{2} \mid \mathcal{G}^{(-i)}, \tilde{\bm{p}}_{i}]\right| \\&\qquad\whprel{1-O(n^{-10})}{\leq} \|\bm{\delta}^{(-i)}\|_\infty\left(2\sqrt{5}\sqrt{2d_i \log(n)} + \frac{20}{3}\log(n)\right). \label{eq:unweighted-var2}
\end{align}

Combining \cref{eq:unweighted-exp2} and \cref{eq:unweighted-var2}, we have
\begin{align}
\circled{2} &\leq \mathbb{E}[\circled{2} \mid \mathcal{G}^{(-i)}, \tilde{\bm{p}}_{i}] + \left|\circled{2} - \mathbb{E}[\circled{2} \mid \mathcal{G}^{(-i)}, \tilde{\bm{p}}_{i}]\right| \\
&\awhprel{1-O(n^{-9})}{\leq}\  \frac{2sd_i}{\sqrt{n}}\left(\circled{c}+ \frac{\|\bm{\pi}\|_\infty}{\sqrt{pk}}\left(\frac{c_3}{\gamma}+8\sqrt{5}(1+h)\sqrt{\log(n)}\right) \right) \\
&\quad +\|\bm{\delta}^{(-i)}\|_\infty\left(2\sqrt{5}\sqrt{2d_i \log(n)} + \frac{20}{3}\log(n)\right) \\
&\leq\frac{2sd_i}{\sqrt{n}}\left(\circled{c}+ \frac{\|\bm{\pi}\|_\infty}{\sqrt{pk}}\left(\frac{c_3}{\gamma}+8\sqrt{5}(1+h)\sqrt{\log(n)}\right) \right) \\ &\quad+\left(\underset{\circled{c}}{\boxed{\|\hat{\bm{\pi}}^{(-i)}-\hat{\bm{\pi}}\|_2}}+\underset{\circled{x}}{\boxed{\|\hat{\bm{\pi}}-\bm{\pi}\|_\infty}}+\underset{\max \frac{\left|\tcircled{a}\right|'}{\tcircled{b}'}}{\boxed{\|\bm{\pi}-\bar{\bm{\pi}}^{(-i)}\|_\infty}}\right) \\
&\qquad\qquad \left(2\sqrt{10}\sqrt{d_i \log(n)} + \frac{20}{3}\log(n)\right) \\
&\awhprel{1-O(n^{-9})}{\leq}\ \frac{2sd_i}{\sqrt{n}}\left(\circled{c}+ \frac{\|\bm{\pi}\|_\infty}{\sqrt{pk}}\left(\frac{c_3}{\gamma}+8\sqrt{5}(1+h)\sqrt{\log(n)}\right) \right) \\
&\quad +\left(\circled{c}+\circled{x}+ 2 (1+h) \|\bm{\pi}\|_\infty \sqrt{\frac{80\log(n)}{npk}} \right)\\
&\qquad\qquad \left(2\sqrt{10}\sqrt{d_i \log(n)} + \frac{20}{3}\log(n)\right).
\end{align}

Next, we bound \circled{c}. From \cite[Theorem 5.1]{Chenetal2019}, we have
\begin{align}
\circled{c} &\leq \sqrt{nh} \|\hat{\bm{\pi}}^{(-i)}-\hat{\bm{\pi}}\|_{\bm{\pi}}\\
&\leq \sqrt{nh} \frac{\boxed{\|\hat{\bm{\pi}}^{{(-i)}\tran} (\hat{S}^{(-i)}-\hat{S})\|_{\bm{\pi}}}\mathrlap{\ \circled{f}}}{1-|\lambda_2(S)|-\boxed{\|\hat{S}-S\|_{\bm{\pi}}}\mathrlap{\ \circled{g}}} \\
& \sqrt{nh} \leq\frac{\circled{f}}{\gamma-\circled{g}}.
\end{align}

First, we bound \circled{g}.
From \cite[Lemma 3]{NegahbanOhShah2017}, we have
\begin{equation}
\circled{g} \leq h \|\hat{S}-S\|_2 \whprel{1-O(n^{-10})}{\leq} 80h \sqrt{\frac{\log(n)}{k d_{\text{max}}}} \quad \whprel{1-O(n^{-10})}{\leq}\  80h \sqrt{\frac{2}{c_0 k}} \leq \frac{\gamma}{2}
\end{equation}
since we assumed that $c_0 \geq 10240 h^2/\gamma^2$ and $k\geq 5$.

To bound \circled{f}, we introduce an intermediate leave-one-out matrix $\hat{S}^{(-m,\mathcal{E})}$ (cf. \cite[Lemma 5.6]{Chenetal2019}). While $\hat{S}^{(-m)}$ is independent of both the graph edges connected to node $m$ and the comparison data regarding $m$, $\hat{S}^{(-m,\mathcal{E})}$ is conditioned on the full comparison graph $\mathcal{E}$, i.e.,
\begin{equation}
    \hat{S}^{(-m,\mathcal{E})}_{ij} \triangleq \begin{cases}
        \hat{S}_{ij}=\frac{\hat{p}_{ij}}{d}, & \text{if } i \neq m \text{ and } j \neq m \text{ and } i\neq j \\
        \frac{p_{ij} A_{ij}}{d}, & \text{if (} i=m \text{ or } j=m \text{) and } i \neq j,\\
        1-\sum_{l:l\neq i} (\hat{S}^{(-m,\mathcal{E})})_{il}, & \text{if } i=j.
    \end{cases}
\end{equation}
This definition allows us to decompose \circled{f} to isolate the influence of the graph edges and the comparison data.
\begin{align}
    &\phantom{{}={}}\circled{f} \\
    &\leq \sqrt{\frac{h}{n}} \|\hat{\bm{\pi}}^{(-i)\tran} (\hat{S}^{(-i)}-\hat{S})\|_2 \\
    &\leq \sqrt{\frac{h}{n}}\left(\underset{\circled{h}}{\boxed{\|\hat{\bm{\pi}}^{(-i)\tran} (\hat{S}^{(-i)}-\hat{S}^{(-i,\mathcal{E})})\|_2}}+\underset{\circled{i}}{\boxed{\|\hat{\bm{\pi}}^{(-i)\tran} (\hat{S}^{(-i,\mathcal{E})}-\hat{S})\|_2}}\right)
\end{align}

First, we bound \circled{h}. Observe that
\begin{align}
    &\phantom{{}={}} (\hat{S}^{(-i)}-\hat{S}^{(-i,\mathcal{E})})_{jl} \\
    &= \begin{cases}
        \frac{p_{jl}(q_{jl}-A_{jl})}{d}, & \text{if (} j=i \text{ or } l=i \text{) and } j\neq l, \\
        \frac{p_{ji}(-q_{ji}+A_{ji})}{d}, & \text{if } j=l \text{ and } j\neq i, \\
        \sum_{m\neq i}\frac{p_{im}(-q_{im}+A_{im})}{d}, &\text{if } j=l=i, \\
        0, & \text{otherwise}.
    \end{cases}
\end{align}
Note that
\begin{align}
&\phantom{{}={}}(\bm{\pi}^\tran(\hat{S}^{(-i)}-\hat{S}^{(-i,\mathcal{E})}))_{j} \\
&= \begin{cases}
    \pi_j \frac{p_{ji}(-q_{ji}+A_{ji})}{d} + \pi_i \frac{p_{ij}(q_{ij}-A_{ij})}{d} = 0, & \text{if } j \neq i, \\
    \sum_{l\neq i} \pi_l \frac{p_{li}(q_{li}-A_{li})}{d} + \pi_i \sum_{l\neq i} \frac{p_{il}(-q_{il}+A_{il})}{d} = 0, & \text{if } j=i,
\end{cases}
\end{align}
thus $\hat{\bm{\pi}}^{(-i)\tran}(\hat{S}^{(-i)}-\hat{S}^{(-i,\mathcal{E})})=(\hat{\bm{\pi}}^{(-i)}-\bm{\pi})^\tran(\hat{S}^{(-i)}-\hat{S}^{(-i,\mathcal{E})})$.
Furthermore, for all $j \neq l$, since $|(\hat{S}^{(-i)}-\hat{S}^{(-i,\mathcal{E})})_{jl}| \leq \frac{1}{d}$ when $A_{jl}=1$ and $|(\hat{S}^{(-i)}-\hat{S}^{(-i,\mathcal{E})})_{jl}| \leq \frac{q_{jl}}{d}$ when $A_{jl}=0$, we have for all $j\neq i$,
\begin{equation}
\left|((\hat{\bm{\pi}}^{(-i)}-\bm{\pi})^\tran(\hat{S}^{(-i)}-\hat{S}^{(-i,\mathcal{E})}))_j\right| \leq \begin{cases}
    \frac{2}{d} \|\hat{\bm{\pi}}^{(-i)}-\bm{\pi}\|_\infty, & \text{if } A_{ji} = 1, \\
    \frac{2p}{d} \|\hat{\bm{\pi}}^{(-i)}-\bm{\pi}\|_\infty, &\text{otherwise.}
\end{cases}
\end{equation}
Finally, for $j=i$,
\begin{align}
&\phantom{{}={}} \left|((\hat{\bm{\pi}}^{(-i)}-\bm{\pi})^\tran(\hat{S}^{(-i)}-\hat{S}^{(-i,\mathcal{E})}))_i\right| \\
&= \left|\sum_{l\neq i} (\hat{\pi}^{(-i)}_l-\pi_l) \frac{p_{li}(q_{li}-A_{li})}{d} + (\hat{\pi}^{(-i)}_i-\pi_i) \sum_{l\neq i} \frac{p_{il}(-q_{il}+A_{il})}{d} \right| \\
&\leq \left|\max_{l\neq i} \frac{(\hat{\pi}^{(-i)}_l-\pi_l)p_{li}-(\hat{\pi}^{(-i)}_i-\pi_i)p_{il}}{d} \sum_{l\neq i} (q_{li}-A_{li}) \right| \\
&\leq \left|\frac{2\|\hat{\bm{\pi}}^{(-i)}-\bm{\pi}\|_\infty}{d} \sum_{l\neq i} (q_{li}-A_{li}) \right| \\
&\awhprel{1-O(n^{-10})}{\leq} \frac{2\|\hat{\bm{\pi}}^{(-i)}-\bm{\pi}\|_\infty}{d} \left(2\sqrt{5}\sqrt{\sum_j q_{ji}(1-q_{ji}) \log(n)} + \frac{20}{3}\log(n)\right) \\
&\awhprel{1-O(n^{-10})}{\leq} \frac{2\|\hat{\bm{\pi}}^{(-i)}-\bm{\pi}\|_\infty}{d} \left(2\sqrt{10}\sqrt{d_i \log(n)} + \frac{20}{3}\log(n)\right).
\end{align}

Thus, we have
\begin{align}
&\phantom{{}={}}\circled{h}\\
&\awhprel{1-O(n^{-10})}{\leq}\ \  \sqrt{\splitfrac{d_i \left(\frac{2}{d} \|\hat{\bm{\pi}}^{(-i)}-\bm{\pi}\|_\infty\right)^2 \!+ (n-d_i-1) \left(\frac{2p}{d} \|\hat{\bm{\pi}}^{(-i)}-\bm{\pi}\|_\infty\right)^2}{+\left(\frac{2\|\hat{\bm{\pi}}^{(-i)}-\bm{\pi}\|_\infty}{d} \left(2\sqrt{10}\sqrt{d_i \log(n)} + \frac{20}{3}\log(n)\right)\right)^2}} \\
&\leq \left(2\sqrt{d_i}+2p\sqrt{n}+4\sqrt{10}\sqrt{d_i \log(n)} + \frac{40}{3}\log(n)\right) \frac{\|\hat{\bm{\pi}}^{(-i)}-\bm{\pi}\|_\infty}{d}.
\end{align}

Next, we bound \circled{i}. Observe that
\begin{align}
    &\phantom{{}={}} (\hat{S}^{(-i,\mathcal{E})}-\hat{S})_{jl} \\
    &= \begin{cases}
        \frac{A_{jl}(p_{jl}-\hat{p}_{jl})}{d}, & \text{if (} j=i \text{ or } l=i \text{) and } j\neq l, \\
        \frac{A_{ji}(-p_{ji}+\hat{p}_{ji})}{d}, & \text{if } j=l \text{ and } j\neq i, \\
        \sum_{m\neq i}\frac{A_{im}(-p_{im}+\hat{p}_{im})}{d}, &\text{if } j=l=i, \\
        0, & \text{otherwise}.
    \end{cases}
\end{align}
Thus, for $j\neq i$, we have
\begin{align}
&\phantom{{}={}} \left|(\hat{\bm{\pi}}^{(-i)\tran} (\hat{S}^{(-i,\mathcal{E})}-\hat{S}))_j\right| \\
&= \left|\hat{\pi}^{(-i)}_j \frac{A_{ji}(-p_{ji}+\hat{p}_{ji})}{d} + \hat{\pi}^{(-i)}_i \frac{A_{ij}(p_{ij}-\hat{p}_{ij})}{d}\right| \\
&\leq \frac{2A_{ij} \|\hat{\bm{\pi}}^{(-i)}\|_\infty}{d} |p_{ij}- \hat{p}_{ij}| \\
&\awhprel{1-O(n^{-10})}{\leq} \frac{2A_{ij}\|\hat{\bm{\pi}}^{(-i)}\|_\infty}{d} \sqrt{\frac{5\log(n)}{k}}.
\end{align}
When $j=i$,
\begin{align}
&\phantom{{}={}} \left|(\hat{\bm{\pi}}^{(-i)\tran} (\hat{S}^{(-i,\mathcal{E})}-\hat{S}))_i\right| \\
&= \left|\sum_{j\neq i} \hat{\pi}^{(-i)}_j \frac{A_{ji}(p_{ji}-\hat{p}_{ji})}{d} + \hat{\pi}^{(-i)}_i \sum_{j\neq i} \frac{A_{ij}(-p_{ij}+\hat{p}_{ij})}{d}\right| \\
&\leq \frac{2 \|\hat{\bm{\pi}}^{(-i)}\|_\infty}{d} \left|\sum_{j\neq i} A_ij(p_{ij}-\hat{p}_{ij})\right| \\
&\awhprel{1-O(n^{-10})}{\leq} \frac{2 \|\hat{\bm{\pi}}^{(-i)}\|_\infty}{d} \sqrt{\frac{5d_i\log(n)}{k}}.
\end{align}
Thus,
\begin{align}
    &\phantom{{}={}}\circled{i} \\
    &\awhprel{1-O(n^{-9})}{\leq} \  \sqrt{d_i \left(\frac{2\|\hat{\bm{\pi}}^{(-i)}\|_\infty}{d} \sqrt{\frac{5\log(n)}{k}}\right)^2\!\! + \left(\frac{2 \|\hat{\bm{\pi}}^{(-i)}\|_\infty}{d} \sqrt{\frac{5d_i\log(n)}{k}}\right)^2}\\
    &= \frac{2\|\hat{\bm{\pi}}^{(-i)}\|_\infty}{d} \sqrt{\frac{10d_i\log(n)}{k}}.
\end{align}

Putting it all together, we have
\begin{align}
&\phantom{{}={}} \circled{c} \\
&\leq \sqrt{nh} \frac{\circled{f}}{\gamma-\circled{g}} \\
&\awhprel{1-O(n^{-9})}{\leq}\ \frac{2\sqrt{nh}}{\gamma}\circled{f} \\
&\leq \frac{2h}{\gamma} \left(\circled{h}+\circled{i}\right) \\
&\leq \frac{2h}{\gamma} \biggl( \frac{2\|\hat{\bm{\pi}}^{(-i)}\|_\infty}{d} \sqrt{\frac{10d_i\log(n)}{k}} \\
&+
\left(2\sqrt{d_i}+2p\sqrt{n}+4\sqrt{10d_i \log(n)} + \frac{40}{3}\log(n)\right) \frac{\|\hat{\bm{\pi}}^{(-i)}-\bm{\pi}\|_\infty}{d} \biggr) \\
&\awhprel{1-O(n^{-10})}{\leq}\ \  \frac{4h}{\gamma} \biggl( \|\hat{\bm{\pi}}^{(-i)}\|_\infty \sqrt{\frac{20\log(n)}{dk}} \\
&\qquad+
\left(\frac{1}{\sqrt{d}}\!+\!\frac{2}{\sqrt{n}}\!+\!\frac{4\sqrt{5}}{\sqrt{c_0}}\!+\!\frac{40}{3c_0}\right) \|\hat{\bm{\pi}}^{(-i)}\!-\bm{\pi}\|_\infty \biggr) \\
&\leq \frac{4h}{\gamma} \biggl( \sqrt{\frac{20\log(n)}{dk}}\|\bm{\pi}\|_\infty\\
&\qquad+
\left(\sqrt{\frac{20}{c_0k}}\!+\!\frac{1}{\sqrt{d}}\!+\!\frac{2}{\sqrt{n}}\!+\!\frac{4\sqrt{5}}{\sqrt{c_0}}\!+\!\frac{40}{3c_0}\right) \|\hat{\bm{\pi}}^{(-i)}\!-\bm{\pi}\|_\infty \biggr) \\
&\leq \frac{8\sqrt{5}h}{\gamma} \sqrt{\frac{\log(n)}{dk}} \|\bm{\pi}\|_\infty\\
&\qquad+ \left(\frac{1}{4\sqrt{10}}+\frac{1}{8\sqrt{5\log(n)}
}+\frac{8h}{\gamma\sqrt{n}}+ \frac{1}{2\sqrt{2}}+\frac{\gamma}{192h} \right) \|\hat{\bm{\pi}}^{(-i)}\!-\bm{\pi}\|_\infty \\
&\leq \frac{8\sqrt{5}h}{\gamma}\sqrt{\frac{\log(n)}{dk}} \|\bm{\pi}\|_\infty + \frac{1}{2} \|\hat{\bm{\pi}}^{(-i)}\!-\bm{\pi}\|_\infty \\
&\leq \frac{8\sqrt{5}h}{\gamma}\sqrt{\frac{\log(n)}{dk}} \|\bm{\pi}\|_\infty + \frac{1}{2}\left(\underset{\circled{c}}{\boxed{\|\hat{\bm{\pi}}^{(-i)}-\hat{\bm{\pi}}\|_2}}+\|\hat{\bm{\pi}}-\bm{\pi}\|_\infty\right).
\end{align}

Thus,
\begin{equation}
\circled{c} \whprel{1-O(n^{-9})}{\leq} \frac{16\sqrt{5}h}{\gamma}\sqrt{\frac{\log(n)}{dk}} \|\bm{\pi}\|_\infty+\|\hat{\bm{\pi}}-\bm{\pi}\|_\infty.
\end{equation}

Finally, combining all the intermediate results, we have
\begin{align}
&\phantom{{}={}}\circled{x} = \|\hat{\bm{\pi}}-\bm{\pi}\|_\infty \\
&= \max_i \left|\delta_i + \frac{\circled{a}}{\circled{b}}\right| \\
&\awhprel{1-O(n^{-9})}{\leq} \max_i |\delta_i| +  2(1+h)\|\bm{\pi}\|_\infty \sqrt{\frac{40\log(n)}{npk}} \\
&= \max_i \left|\frac{\circled{1}+\circled{2}+\circled{3}}{\circled{b}}\right| +  2(1+h)\|\bm{\pi}\|_\infty \sqrt{\frac{40\log(n)}{npk}} \\
&\awhprel{1-O(n^{-8})}{\leq} \max_i \frac{1}{\circled{b}} \biggg(\sqrt{d_i}\ \circled{c}\\
&\qquad+ \frac{2sd_i}{\sqrt{n}}\left(\circled{c}+ \frac{\|\bm{\pi}\|_\infty}{\sqrt{pk}}\left(\frac{c_3}{\gamma}+8\sqrt{5}(1+h)\sqrt{\log(n)}\right) \right) \\
&\qquad +\left(\circled{c}+\circled{x}+ 2 (1+h) \|\bm{\pi}\|_\infty \sqrt{\frac{80\log(n)}{npk}} \right)\\
&\qquad\qquad\qquad \left(2\sqrt{10}\sqrt{d_i \log(n)} + \frac{20}{3}\log(n)\right)\\
&\qquad + 2 \ \circled{b} \ (1+h) \|\bm{\pi}\|_\infty \sqrt{\frac{80\log(n)}{npk}} \biggg) \\
&\awhprel{1-O(n^{-9})}{\leq} 2 (1+h) \|\bm{\pi}\|_\infty \sqrt{\frac{80\log(n)}{npk}}\\
&+ 2(1+h)\max_i \biggg(\frac{\circled{c}}{\sqrt{d_i}}+ \frac{2s}{\sqrt{n}}\left(\circled{c}+ \frac{\|\bm{\pi}\|_\infty}{\sqrt{pk}}\left(\frac{c_3}{\gamma}+8\sqrt{5}(1+h)\sqrt{\log(n)}\right) \right)\\
&\qquad +\left(\circled{c}+\circled{x}+ 2 (1+h) \|\bm{\pi}\|_\infty \sqrt{\frac{80\log(n)}{npk}} \right)\\
&\qquad\qquad\qquad \frac{2\sqrt{10}\sqrt{d_i \log(n)} + \frac{20}{3}\log(n)}{d_i}\biggg) \\
&\awhprel{1-O(n^{-9})}{\leq}2 (1+h) \|\bm{\pi}\|_\infty \sqrt{\frac{80\log(n)}{npk}}\\
&+ 2(1+h)\max_i \biggg(\frac{\circled{c}}{\sqrt{d_i}}+ \frac{2s}{\sqrt{n}}\left(\circled{c}+ \frac{\|\bm{\pi}\|_\infty}{\sqrt{pk}}\left(\frac{c_3}{\gamma}+8\sqrt{5}(1+h)\sqrt{\log(n)}\right) \right)\\
&\qquad +\left(\circled{c}+\circled{x}+ 2 (1+h) \|\bm{\pi}\|_\infty \sqrt{\frac{80\log(n)}{npk}} \right)\left(\frac{4\sqrt{5}}{\sqrt{c_0}} + \frac{40}{3c_0}\right)\biggg) \\
&\leq \left(8\sqrt{5}+32\sqrt{5}(1+h)s+\frac{16\gamma}{\sqrt{5}}\right) (1+h) \|\bm{\pi}\|_\infty \sqrt{\frac{\log(n)}{npk}} \\
&\qquad + \frac{4(1+h)sc_3\|\bm{\pi}\|_\infty}{\gamma\sqrt{npk}}+\frac{2\gamma\ \circled{x}}{5}\\
&\qquad + \max_i \left( \left(\frac{2(1+h)}{\sqrt{d_i}}+\frac{4s(1+h)}{\sqrt{n}}+\frac{2\gamma}{5}\right) \circled{c}\right)\\
&\awhprel{1-O(n^{-8})}{\leq}\quad \left(8\sqrt{5}+32\sqrt{5}(1+h)s+\frac{16\gamma}{\sqrt{5}}\right) (1+h) \|\bm{\pi}\|_\infty \sqrt{\frac{\log(n)}{npk}} \\
&\qquad + \frac{4(1+h)sc_3\|\bm{\pi}\|_\infty}{\gamma\sqrt{npk}}+\frac{2\gamma\ \circled{x}}{5}\\
&\qquad +  \left(\frac{2\sqrt{2}(1+h)}{\sqrt{c_0\log(n)}}+\frac{4s(1+h)}{\sqrt{n}}+\frac{2\gamma}{5}\right)\\
&\qquad\qquad\qquad\left(\frac{16\sqrt{10}h}{\gamma}\sqrt{\frac{\log(n)}{npk}} \|\bm{\pi}\|_\infty+\circled{x}\right) \\
&=\overbracket{\left(\begin{matrix}8\sqrt{5}(1+h)+32\sqrt{5}(1+h)^2s+\frac{16(1+h)\gamma}{\sqrt{5}}\\+\frac{4(1+h)sc_3}{\gamma \log(n)}+\frac{64\sqrt{5}h(1+h)}{\gamma \sqrt{c_0\log(n)}}+\frac{64\sqrt{10}h(1+h)s}{\sqrt{n}}+\frac{64h}{\sqrt{10}}\end{matrix}\right)}^{\circled{A}} \sqrt{\frac{\log(n)}{npk}} \|\bm{\pi}\|_\infty \\
&\qquad + \underbracket{\left(\frac{4\gamma}{5}+\frac{2\sqrt{2}(1+h)}{\sqrt{c_0 \log(n)}}+\frac{4s(1+h)}{\sqrt{n}}\right)}_{\circled{B}} \ \circled{x}.
\end{align}
Observe that for sufficiently large $n$,
\begin{equation}
\circled{A} \leq \frac{1}{\gamma} \underbracket{\left((1+h)(4sc_3+64\sqrt{5}h)\right)}_{\circled{C}} + \underbracket{32\sqrt{5}(1+h)+128\sqrt{10} (1+h)^2s}_{\circled{D}},
\end{equation}
\begin{equation}
\circled{B} \leq \frac{9}{10}.
\end{equation}
Thus,
\begin{equation}
\frac{\|\hat{\bm{\pi}}-\bm{\pi}\|_\infty}{\|\bm{\pi}\|_\infty}\ \whprel{1-O(n^{-8})}{\leq}\quad \left(\frac{c_1}{\gamma} +c_2\right) \sqrt{\frac{\log(n)}{npk}}
\end{equation}
where the constants $c_1=10\circled{C}$ and $c_2=10\circled{D}$ only depend on $h$, $s$, and $c_3$ but not $n$ or $p$.
\end{proof}

\section[Proof of Theorem 4.1]{Proof of \cref{thm:error-bounds-for-weighted-rank-centrality}} \label{app:proof-error-bounds-for-weighted-rank-centrality}
\begin{proof}
First, recall that
\begin{equation}
    \hat{\pi}_i = \frac{\sum_{j: j\neq i} \hat{p}_{ji} w_{ji} \hat{\pi}_j}{\sum_{j:j\neq i} \hat{p}_{ij} w_{ij}}.
\end{equation}
As in the proof of \cref{thm:entry-wise-error-bounds-for-unweighted-rank-centrality}, we define the auxiliary vector $\bar{\bm{\pi}}$ as
\begin{equation}
    \bar{\pi}_i \triangleq \frac{\sum_{j: j\neq i} \hat{p}_{ji} w_{ji} \pi_j}{\sum_{j:j\neq i} \hat{p}_{ij} w_{ij}}.
\end{equation}

Next, we decompose the entry-wise error as
\begin{align}
\frac{\hat{\pi}_i-\pi_i}{\pi_i} &= \frac{\bar{\pi}_i - \pi_i}{\pi_i} + \frac{\hat{\pi}_i-\bar{\pi}_i}{\pi_i} \\
&\labelrel{=}{rel:weighted-dec1} \frac{\boxed{\textstyle \sum_{j:j\neq i} w_{ij} (\hat{p}_{ji} \pi_j - \hat{p}_{ij}\pi_i)}\ \mathrlap{\circled{a}}}{\boxed{\textstyle \pi_i \sum_{j:j\neq i} w_{ij} \hat{p}_{ij}}\ \mathrlap{\circled{b}}}+\boxed{\frac{\hat{\pi}_i-\bar{\pi}_i}{\pi_i}}\ \circled{c}
\end{align}
where \cref{rel:weighted-dec1} follows from $w_{ij}=w_{ji}$.

First, we upper bound the magnitude of \circled{a} with high probability.
\begin{align}
&\phantom{{}={}}\mathbb{P}\left(|\circled{a}| \leq \|\bm{\pi}\|_\infty \sqrt{\frac{40n\log(n)p}{k}}\right) \\
&\labelrel{\geq}{rel:weighted-hoeffding1} 1 - 2 \exp\left(-\frac{2\left(\|\bm{\pi}\|_\infty \sqrt{\frac{40n\log(n)p}{k}}\right)^2}{\sum_{j: j\neq i} w_{ij}^2 (\pi_i+\pi_j)^2/k}\right) \\
&\geq 1 - 2\exp \left(-\frac{20 n\log(n) p}{\sum_{j:j\neq i} w_{ij}^2}\right) \\
&\labelrel{\geq}{rel:weighted-w2sum} 1 - 2 \exp\left(-10\log(n)\right) = 1 - \frac{2}{n^{10}},
\end{align}
where \cref{rel:weighted-hoeffding1} is due to Hoeffding's inequality and \cref{rel:weighted-w2sum} is due to $\sum_{j: j\neq i} w^2_{ij} \leq \sum_{j: j\neq i} w_{ij} \leq 2np$.

Next, to lower bound \circled{b}, we first define the $\bm{\pi}$-weighted Laplacian matrix $L_{\bm{\pi}}^W$ and the corresponding quantity $d_\text{min}^{\bm{\pi}}$,
\begin{align}
    L_{\bm{\pi}}^W &\triangleq \sum_{i<j\leq n} \frac{\pi_i\pi_j}{\pi_i+\pi_j} w_{ij} (\bm{e}_i-\bm{e}_j)(\bm{e}_i-\bm{e}_j)^\tran, \\
    d_\text{min}^{\bm{\pi}} &\triangleq \min_{i\in[n]} \sum_{j=1}^n \frac{\pi_i \pi_j}{\pi_i+\pi_j} w_{ij}.
\end{align}
Note that since
\begin{equation}
    \frac{\|\bm{\pi}\|_\infty}{2h} \leq \frac{\pi_i\pi_j}{\pi_i+\pi_j} \leq \frac{\|\bm{\pi}\|_\infty}{2},
\end{equation}
we have that (cf. \cite[Section 4.1]{MonotoneAdversary})
\begin{equation}
\frac{\|\bm{\pi}\|_\infty}{2h} \lambda_{n-1}(L^W) \leq 
\lambda_{n-1}(L_{\bm{\pi}}^W) \leq \frac{\|\bm{\pi}\|_\infty}{2} \lambda_{n-1}(L^W).\label{eq:lambda-lwpi-upper}
\end{equation}
Then,
\begin{align}
\circled{b} &= \underbracket{\sum_{j=1}^n \frac{\pi_i\pi_j}{\pi_i+\pi_j} w_{ij}}_{\geq d_\text{min}^{\bm{\pi}}} + \underbracket{\pi_i \sum_{j=1}^n w_{ij} (\hat{p}_{ij}-p_{ij})}_{\circled{d}},
\end{align}
and
\begin{align}
&\phantom{{}={}}\mathbb{P}\left(\circled{d} \geq -\|\bm{\pi}\|_\infty \sqrt{\frac{10n\log(n)p}{k}} \right) \\&\geq 1 - \exp\left(-\frac{2\left(\|\bm{\pi}\|_\infty \sqrt{\frac{10n\log(n)p}{k}}\right)^2}{\pi_i^2 \sum_{j=1}^n w_{ij}^2 / k}\right) \geq 1-\frac{1}{n^{10}} \label{eq:weighted-dbound}
\end{align}
due to Hoeffding's inequality. Thus, with high probability, we have
\begin{align}
\circled{b}\  &\awhprel{1-O(n^{-10})}{\geq} d_\text{min}^{\bm{\pi}} - \|\bm{\pi}\|_\infty \sqrt{\frac{10n\log(n)p}{k}} \\
&\labelrel{\geq}{rel:weighted-k-ass} d_\text{min}^{\bm{\pi}} - \|\bm{\pi}\|_\infty \sqrt{\frac{10n\log(n)p \lambda_{n-1}(L^W)^2}{640 h^2n\log(n) p}} \\
&\labelrel{\geq}{rel:weighted-dminbound} \frac{\lambda_{n-1}(L_{\bm{\pi}}^W)}{2} - \frac{\|\bm{\pi}\|_\infty \lambda_{n-1}(L^W)}{8h} \\
&\labelrel{\geq}{rel:weighted-lambdabound}
\frac{\lambda_{n-1}(L_{\bm{\pi}}^W)}{4}
\end{align}
where \cref{rel:weighted-k-ass} follows from the $k$ assumption, \cref{rel:weighted-dminbound} is due to $\lambda_{n-1}(L_{\bm{\pi}}^W) \leq 2d_\text{min}^{\bm{\pi}}$ from \cite[Lemma 18]{MonotoneAdversary}, and \cref{rel:weighted-lambdabound} follows from \cref{eq:lambda-lwpi-upper}.

Next, we focus on \circled{c}, which we will denote with $\delta_i$. Note that
\begin{align}
\circled{c} &= \delta_i = \frac{\hat{\pi}_i-\bar{\pi}_i}{\pi_i} = \frac{\sum_j w_{ji} \hat{p}_{ji} (\hat{\pi}_j-\pi_j)}{\pi_i\sum_j w_{ij}\hat{p}_{ij}} \\
&= \frac{\sum_j w_{ji} \hat{p}_{ji} \pi_j \delta_j}{\pi_i\sum_j w_{ij}\hat{p}_{ij} } + \frac{\sum_j w_{ji} \hat{p}_{ji} (\bar{\pi}_j-\pi_j)}{\pi_i \sum_j w_{ij} \hat{p}_{ij}},
\end{align}
thus
\begin{equation}
(\pi_i \sum_j w_{ij} \hat{p}_{ij})\delta_i - \sum_{j:j\neq i} w_{ji}\hat{p}_{ji}\pi_j \delta_j = \sum_j w_{ji} \hat{p}_{ji} (\bar{\pi}_j-\pi_j).
\end{equation}
This can be re-expressed as
\begin{equation}
\hat{L}_{\bm{\pi}}^W \bm{\delta} = \bm{r}, \label{eq:weighted-Ld-r}
\end{equation}
where
\begin{align}
(\hat{L}_{\bm{\pi}}^W)_{ij} &\triangleq \begin{cases}
    -w_{ji}\hat{p}_{ji}\pi_j, & \text{if } i \neq j, \\
    \pi_i \sum_k w_{ik} \hat{p}_{ik}, & \text{otherwise},
\end{cases} \\
r_i &\triangleq \sum_j w_{ji} \hat{p}_{ji} (\bar{\pi}_j-\pi_j).
\end{align}
Note that $\mathbb{E}[\hat{L}_{\bm{\pi}}^W] = L_{\bm{\pi}}^W$ and $\hat{L}_{\bm{\pi}}^W$ is column-stochastic (but not necessarily row-stochastic, unlike $L_{\bm{\pi}}^W$). Next, we use Bernstein's inequality for matrices \cite{Tropp2015} to bound $\|\hat{L}_{\bm{\pi}}^W-L_{\bm{\pi}}^W\|_2$ with high probability. Let $r_{ijl}$ be the $l\in[k]$-th comparison result for the pair $(i, j)$ with $r_{ijl} \sim \mathsf{Bernoulli}(p_{ij})$, $r_{jil}=1-r_{ijl}$, and $\sum_{l\in[k]} r_{ijl}/k = \hat{p}_{ij}$. For all $(i, j) \in \mathcal{E}$ with $i>j$ and $l\in[k]$, define the matrix $X_{jil}$ to be
\begin{equation}
X_{jil} \triangleq \frac{1}{k} w_{ji} (r_{jil}-p_{ji}) (\bm{e}_j-\bm{e}_i) (\pi_j \bm{e}_j-\pi_i \bm{e}_i)^\tran.
\end{equation}
Observe that $\mathbb{E}[X_{jil}]=0_{n\times n}$,
\begin{align}
\|\mathbb{E} [X_{jil}^\tran X_{jil}]\|_2 &= \|\mathbb{E} [X_{jil} X_{jil}^\tran]\|_2 = \frac{2}{k^2} w_{ji}^2 p_{ji}(1-p_{ji}) (\pi_i^2+\pi_j^2) \\
&\leq \frac{w_{ji}^2\|\bm{\pi}\|_\infty^2}{k^2},
\end{align}
\begin{equation}
\|X_{jil}\|_2 = \frac{1}{k}w_{ji}|r_{jil}-p_{ji}| \sqrt{2(\pi_i^2+\pi_j^2)} \leq \frac{2 \|\bm{\pi}\|_\infty}{k},
\end{equation}
and
\begin{equation}
\sum_{(i, j)\in\mathcal{E},i>j} \sum_{l\in[k]} X_{jil} = \hat{L}_{\bm{\pi}}^W-L_{\bm{\pi}}^W.
\end{equation}

Thus,
\begin{align}
&\phantom{{}={}}\mathbb{P}\left(\|\hat{L}_{\bm{\pi}}^W-L_{\bm{\pi}}^W\|_2 \geq \|\bm{\pi}\|_\infty \sqrt{\frac{40n^2p\log(n)}{k}}\right) \\
&\leq 2n \exp \left(\frac{-\|\bm{\pi}\|_\infty^2 \frac{40n^2p\log(n)}{2k}}{\text{var}(\hat{L}_{\bm{\pi}}^W-L_{\bm{\pi}}^W)+\frac{2\|\bm{\pi}\|_\infty^2}{3k}\sqrt{\frac{n^2p\log(n)}{k}}}\right) \\
&\leq 2n \exp \left(\frac{-20n^2p\log(n)}{n^2p+\frac{2}{3} \sqrt{\frac{n^2p\log(n)}{k}}}\right) \\
&\leq 2n \exp(-10\log(n)) \leq 2n^{-9}.
\end{align}

Thus, from the assumption on $k$, we have
\begin{equation}
\|\hat{L}_{\bm{\pi}}^W-L_{\bm{\pi}}^W\|_2 \whprel{1-O(n^{-9})}{\leq}\ \  \frac{h}{n}\sqrt{\frac{40n^2p\log(n)}{k}} \leq \frac{\lambda_{n-1}(L_{\bm{\pi}}^W)}{2}.
\end{equation}

Therefore, by taking the $\ell^2$ norm on both sides of \cref{eq:weighted-Ld-r}, we get
\begin{align}
\|\bm{r}\|_2 &= \|\hat{L}_{\bm{\pi}}^W \bm{\delta}\|_2 \geq \|L_{\bm{\pi}}^W \bm{\delta}\|_2-\|\hat{L}_{\bm{\pi}}^W-L_{\bm{\pi}}^W\|_2 \|\bm{\delta}\|_2 \\
&\awhprel{1-O(n^{-9})}{\geq} \lambda_{n-1}(L_{\bm{\pi}}^W) \left(\|\bm{\delta}\|_2-\frac{|\bm{\delta}^\tran \bm{1}_n|}{\sqrt{n}}\right) - \frac{\lambda_{n-1}(L_{\bm{\pi}}^W)}{2}\|\bm{\delta}\|_2 \\
&= \lambda_{n-1}(L_{\bm{\pi}}^W) \left(\frac{\|\bm{\delta}\|_2}{2}-\frac{|\bm{\delta}^\tran \bm{1}_n|}{\sqrt{n}}\right),
\end{align}
thus
\begin{equation}
\|\bm{\delta}\|_2 \whprel{1-O(n^{-9})}{\leq} 2\left(\frac{\|\bm{r}\|_2}{\lambda_{n-1}(L_{\bm{\pi}}^W)}+\frac{|\bm{\delta}^\tran \bm{1}_n|}{\sqrt{n}}\right).
\end{equation}

Next, we bound each element of $\bm{r}$.
\begin{align}
r_i  &= \sum_j w_{ji} \hat{p}_{ji} (\bar{\pi}_j-\pi_j) \\
&= \underbracket{\sum_j w_{ji} \hat{p}_{ji} (\bar{\pi}_j-\bar{\pi}_j^{(-i)})}_{\circled{d}} + \underbracket{\sum_j w_{ji} \hat{p}_{ji} (\bar{\pi}_j^{(-i)}-\pi_j)}_{\circled{e}},
\end{align}
where $\bar{\pi}_j^{(-i)}$ is the leave-one-out version of $\bar{\pi}_j$,
\begin{equation}
\bar{\pi}_j^{(-i)} \triangleq \frac{\sum_{l \notin \{i,j\}} w_{lj} \hat{p}_{lj} \pi_l + w_{ij} p_{ij} \pi_i}{\sum_{l \notin \{i,j\}} w_{jl} \hat{p}_{jl}+w_{ji} p_{ji}}.
\end{equation}

First, note the identity
\begin{equation}
\frac{\circled{1}+\circled{2}}{\circled{3}+\circled{4}}-\frac{\circled{1}+\circled{2}'}{\circled{3}+\circled{4}'} = \frac{\circled{2}-\circled{2}'}{\circled{3}+\circled{4}} + \frac{(\circled{1}+\circled{2}')(\circled{4}'-\circled{4})}{(\circled{3}+\circled{4})(\circled{3}+\circled{4}')}.
\end{equation}
Thus,
\begin{align}
&\phantom{{}={}} \bar{\pi}_j^{(-i)} - \bar{\pi}_j \\
&= \frac{\overbracket{\sum_{l \notin \{i,j\}} w_{lj} \hat{p}_{lj} \pi_l}^{\circled{1}} + \overbracket{w_{ij} p_{ij} \pi_i}^{\circled{2}}}{\underbracket{\sum_{l \notin \{i,j\}} w_{jl} \hat{p}_{jl}}_{\circled{3}}+\underbracket{w_{ji} p_{ji}}_{\circled{4}}}-\frac{\overbracket{\sum_{l \notin \{i,j\}} w_{lj} \hat{p}_{lj} \pi_l}^{\circled{1}} + \overbracket{w_{ij} \hat{p}_{ij} \pi_i}^{\circled{2}'}}{\underbracket{\sum_{l \notin \{i,j\}} w_{jl} \hat{p}_{jl}}_{\circled{3}}+\underbracket{w_{ji} \hat{p}_{ji}}_{\circled{4}'}} \\
&= \frac{w_{ij}\pi_i(p_{ij}-\hat{p}_{ij})}{\sum_{l \notin \{i,j\}} w_{jl} \hat{p}_{jl}+w_{ji} p_{ji}}+ \bar{\pi}_j\frac{w_{ji}(\hat{p}_{ji}-p_{ji})}{\sum_{l \notin \{i,j\}} w_{jl} \hat{p}_{jl}+w_{ji} p_{ji}} \\
&= \frac{w_{ij}\pi_j(\pi_i-\pi_j+\pi_j-\bar{\pi}_j)(p_{ij}-\hat{p}_{ij})}{\underset{\circled{b}'}{\boxed{\textstyle\pi_j \sum_{l \notin \{i,j\}} w_{jl} \hat{p}_{jl}+w_{ji} p_{ji}}}} \\
&\awhprel{1-O(n^{-10})}{\leq}\quad \frac{4w_{ij}\|\bm{\pi}\|_\infty^2(p_{ij}-\hat{p}_{ij})}{\lambda_{n-1}(L_{\bm{\pi}}^W)}\left(1 + \boxed{\frac{\pi_j-\bar{\pi}_j}{\pi_j}}\right) \\
&\awhprel{1-O(n^{-10})}{\leq}\quad \frac{4w_{ij}\|\bm{\pi}\|_\infty^2(p_{ij}-\hat{p}_{ij})}{\lambda_{n-1}(L_{\bm{\pi}}^W)}\left(1+\frac{4\|\bm{\pi}\|_\infty \sqrt{\frac{40n\log(n)p}{k}}}{\lambda_{n-1}(L_{\bm{\pi}}^W)}\right) \\
&\leq \frac{16h w_{ij}\|\bm{\pi}\|_\infty(p_{ij}-\hat{p}_{ij})}{\lambda_{n-1}(L^W)},
\end{align}
where $\circled{b}'$ is a leave-one-out version of \circled{b} and can be bounded similarly.
Thus,
\begin{align}
\circled{d}\  & \awhprel{1-O(n^{-9})}{\leq}\quad \sum_j \frac{16h w_{ij}\|\bm{\pi}\|_\infty(p_{ij}-\hat{p}_{ij})}{\lambda_{n-1}(L^W)} \\
&\awhprel{1-O(n^{-10})}{\leq} \quad \frac{16h \|\bm{\pi}\|_\infty\sqrt{\frac{10n\log(n)p}{k}}}{\lambda_{n-1}(L^W)}.
\end{align}

Next,
\begin{align}
\circled{e} &= \sum_j w_{ji} \hat{p}_{ji} \frac{\sum_{l \notin \{i,j\}} w_{lj} (\hat{p}_{lj} \pi_l-\hat{p}_{jl} \pi_j)}{\sum_{l \notin \{i,j\}} w_{jl} \hat{p}_{jl}+w_{ji} p_{ji}} \\
&\awhprel{1-O(n^{-9})}{\leq}  \frac{4\|\bm{\pi}\|_\infty \overbracket{\textstyle \sum_j \sum_{l \notin \{i,j\}}w_{ji} w_{lj} (\hat{p}_{lj} \pi_l-\hat{p}_{jl} \pi_j)}^{\circled{f}}}{\lambda_{n-1}(L_{\bm{\pi}}^W)}.
\end{align}
Using Hoeffding's inequality,
\begin{align}
&\phantom{{}={}}\mathbb{P}\left(|\circled{f}| \leq \|\bm{\pi}\|_\infty np \sqrt{\frac{80\log(n)}{k}}\right) \\
&\geq 1 - 2 \exp\left(-\frac{2\left(\|\bm{\pi}\|_\infty np \sqrt{\frac{80\log(n)}{k}}\right)^2}{\sum_j \sum_{l \notin \{i,j\}} w_{ji}^2 w_{lj}^2 (\pi_l+\pi_j)^2/k}\right) \\
&\geq 1 - 2\exp \left(-\frac{40 n^2p^2\log(n) }{\sum_j \sum_{l \notin \{i,j\}} w_{ji}^2 w_{lj}^2}\right) \\
&\geq 1 - 2 \exp\left(-10\log(n)\right) = 1 - \frac{2}{n^{10}},
\end{align}

Thus,
\begin{equation}
\circled{e} \whprel{1-O(n^{-9})}{\leq} \frac{16h\|\bm{\pi}\|_\infty np \sqrt{\frac{20\log(n)}{k}}}{\lambda_{n-1}(L^W)},
\end{equation}
which implies
\begin{align}
\|\bm{r}\|_2 &\leq \sqrt{n} \|\bm{r}\|_\infty \\
&\awhprel{1-O(n^{-8})}{\leq}\sqrt{n}(1+\sqrt{2np})\frac{16h \|\bm{\pi}\|_\infty\sqrt{\frac{10n\log(n)p}{k}}}{\lambda_{n-1}(L^W)}.
\end{align}

Now we bound $|\bm{\delta}^\tran \bm{1}_n|$.
\begin{align}
\left|\sum_i \delta_i \right| &\leq hn \left|\sum_i \pi_i \delta_i\right|\\
&=hn\left|\sum_i \hat{\pi}_i-\bar{\pi}_i\right| \\
&=hn\left|\sum_i \pi_i-\bar{\pi}_i\right| \\
&\leq hn \|\bm{\pi}\|_\infty \left|\sum_i \frac{ \sum_{j:j\neq i} w_{ij} (\hat{p}_{ji} \pi_j - \hat{p}_{ij}\pi_i)}{ \pi_i\sum_{j:j\neq i} w_{ij} \hat{p}_{ij}}\right| \\
&\awhprel{1-O(n^{-9})}{\leq} 2h^2n \left| \frac{ \overbracket{\textstyle\sum_i\sum_{j:j\neq i} w_{ij} (\hat{p}_{ji} \pi_j - \hat{p}_{ij}\pi_i)}^{\circled{g}}}{ \lambda_{n-1}(L^W)}\right|,
\end{align}
and using Hoeffding's inequality,
\begin{align}
&\phantom{{}={}}\mathbb{P}\left(|\circled{g}| \leq \|\bm{\pi}\|_\infty n\sqrt{\frac{40\log(n)p}{k}}\right) \\
&\geq 1 - 2 \exp\left(-\frac{2\left(\|\bm{\pi}\|_\infty n \sqrt{\frac{40\log(n)p}{k}}\right)^2}{\sum_i\sum_{j:j\neq i} w_{ij}^2 (\pi_l+\pi_j)^2/k}\right) \\
&\geq 1 - 2\exp \left(-\frac{20 n^2p\log(n) }{\sum_i\sum_{j:j\neq i} w_{ij}^2}\right) \\
&\geq 1 - 2 \exp\left(-10\log(n)\right) = 1 - \frac{2}{n^{10}},
\end{align}
thus
\begin{equation}
|\bm{\delta}^\tran \bm{1}_n|\  \awhprel{1-O(n^{-9})}{\leq} 2n^2h^2 \|\bm{\pi}\|_\infty \frac{\sqrt{\frac{40\log(n)p}{k}}}{\lambda_{n-1}(L^W)}.
\end{equation}

Therefore,
\begin{align}
\|\bm{\delta}\|_2\  &\awhprel{1-O(n^{-9})}{\leq} 2\left(\frac{\|\bm{r}\|_2}{\lambda_{n-1}(L_{\bm{\pi}}^W)}+\frac{|\bm{\delta}^\tran \bm{1}_n|}{\sqrt{n}}\right) \\
&\leq 2\left(2h\frac{\|\bm{r}\|_2}{\|\bm{\pi}\|_\infty\lambda_{n-1}(L^W)}+\frac{|\bm{\delta}^\tran \bm{1}_n|}{\sqrt{n}}\right) \\
&\awhprel{1-O(n^{-8})}{\leq} 2\Biggl(2h\frac{16h\sqrt{n}(1+\sqrt{2np}) \sqrt{\frac{10n\log(n)p}{k}}}{\lambda_{n-1}(L^W)^2}\\
&\qquad +2nh^2 \|\bm{\pi}\|_\infty \frac{\sqrt{\frac{40n\log(n)p}{k}}}{\lambda_{n-1}(L^W)}\Biggr) \\
&\leq 4nh^2\Biggl(\frac{32p \sqrt{\frac{20n\log(n)}{k}}}{\lambda_{n-1}(L^W)^2}+ \frac{\|\bm{\pi}\|_\infty\sqrt{\frac{40n\log(n)p}{k}}}{\lambda_{n-1}(L^W)}\Biggr).
\end{align}

Next, since
\begin{align}
\delta_i &= \frac{\sum_j w_{ji} \hat{p}_{ji} \pi_j \delta_j}{\pi_i\sum_j w_{ij}\hat{p}_{ij} } + \frac{\sum_j w_{ji} \hat{p}_{ji} (\bar{\pi}_j-\pi_j)}{\pi_i \sum_j w_{ij} \hat{p}_{ij}} \\
&\awhprel{1-O(n^{-10})}{\leq}\ \ \ \frac{8h}{\|\bm{\pi}\|_\infty \lambda_{n-1}(L^W)} \left(\sqrt{\sum_j w_{ji}^2 \hat{p}_{ji}^2 \pi_j^2} \|\bm{\delta}\|_2 + r_i\right) \\
&\awhprel{1-O(n^{-8})}{\leq} \ \  \frac{8h}{\|\bm{\pi}\|_\infty \lambda_{n-1}(L^W)} \Biggl(\\
&\qquad\sqrt{np} \|\bm{\pi}\|_\infty  4nh^2\biggl(\frac{32p \sqrt{\frac{20n\log(n)}{k}}}{\lambda_{n-1}(L^W)^2}+ \frac{\|\bm{\pi}\|_\infty\sqrt{\frac{40n\log(n)p}{k}}}{\lambda_{n-1}(L^W)}\biggr) \\
&\qquad + \frac{32hnp\|\bm{\pi}\|_\infty \sqrt{\frac{20\log(n)}{k}}}{\lambda_{n-1}(L^W)}\Biggr) \\
&= \frac{1024n^2h^3p \sqrt{\frac{20\log(n)p}{k}}}{\lambda_{n-1}(L^W)^3} + \frac{32 h^4np \sqrt{\frac{40\log(n)}{k}}}{\lambda_{n-1}(L^W)^2} \\
&\qquad + \frac{256h^2np \sqrt{\frac{20\log(n)}{k}}}{\lambda_{n-1}(L^W)^2}.
\end{align}

Finally, we have
\begin{align}
&\phantom{{}={}}\frac{\|\hat{\bm{\pi}}-\bm{\pi}\|_\infty}{\|\bm{\pi}\|_\infty} \\
&\awhprel{1-O(n^{-9})}{\leq} \frac{8h \sqrt{\frac{40n\log(n)p}{k}}}{\lambda_{n-1}(L^W)} + \|\bm{\delta}\|_\infty \\
&\awhprel{1-O(n^{-7})}{\leq}\ \ \frac{\overbracket{16\sqrt{10}h}^{c_1}}{\lambda_{n-1}(L^W)}\sqrt{\frac{n\log(n)p}{k}} + \frac{\overbracket{2048\sqrt{5}h^3}^{c_3}}{\lambda_{n-1}(L^W)^3} \sqrt{\frac{n^4\log(n)p^3}{k}} \\
&\qquad +\frac{\overbracket{64\sqrt{10}h^4+512\sqrt{5}h^2}^{c_2}}{\lambda_{n-1}(L^W)^2}\sqrt{\frac{n^2\log(n)p^2}{k}},
\end{align}
completing the proof.
\end{proof}

\end{document}